%% file: GuaranteesMRLE.tex
\documentclass{arXiv}
\usepackage{amsmath}

\usepackage{times}
\usepackage{bm}
\usepackage{natbib}

\usepackage[plain,noend]{algorithm2e}

\makeatletter
\renewcommand{\algocf@captiontext}[2]{#1\algocf@typo. \AlCapFnt{}#2} 
\def\@algocf@capt@plain{top}
\renewcommand{\algocf@makecaption}[2]{%
  \addtolength{\hsize}{\algomargin}%
  \sbox\@tempboxa{\algocf@captiontext{#1}{#2}}%
  \ifdim\wd\@tempboxa >\hsize
    \hskip .5\algomargin%
    \parbox[t]{\hsize}{\algocf@captiontext{#1}{#2}}
  \else%
    \global\@minipagefalse%
    \hbox to\hsize{\box\@tempboxa}
  \fi%
  \addtolength{\hsize}{-\algomargin}%
}
\makeatother

\input{Sections/SpecialNotation}

\usepackage{xcolor} 
\begin{document}

\jname{Biometrika}
\jyear{2017}
\jvol{xx}
\jnum{x}


\markboth{R. Zhuang \and J. Lederer}{Maximum Regularized Likelihood Estimators}


\title{\vspace{-1cm}Maximum Regularized Likelihood Estimators:\\
	A General Prediction Theory and Applications}

\author{Rui Zhuang}
\affil{Department of Biostatistics, University of Washington,\\ 1705 NE Pacific St, Seattle, WA 98195, USA \email{rui2@uw.edu}}

\author{\and Johannes Lederer}
\affil{Department of Mathematics, Ruhr-University Bochum,\\ 44780 Bochum, Germany \email{johannes.lederer@rub.de}}

\maketitle

\begin{abstract}
Maximum regularized likelihood estimators (MRLEs) are arguably the most established class of estimators in high-dimensional statistics. In this paper, we derive guarantees for MRLEs in Kullback-Leibler divergence, a general measure of prediction accuracy. We assume only that the densities have a convex parametrization and that the regularization is definite and positive homogenous. The results thus apply  to a very large variety of models and estimators, such as tensor regression and graphical models with convex and non-convex regularized methods. A main conclusion is that MRLEs are broadly consistent in prediction - regardless of whether restricted eigenvalues or similar conditions hold.
\end{abstract}


\input{Sections/Intro}

\input{Sections/GeneralTheory}

\input{Sections/Examples}

\input{Sections/Discussion}

\input{Sections/Acknowledgment}

\bibliographystyle{biometrika}

\input{GuaranteesMRLE.bbl}

\appendix

\appendixone
\input{Sections/MainProofs}

\appendixtwo
\input{Sections/OtherProofs}

\appendixthree
\input{Sections/EP}

\appendixfour
\input{Sections/Appendix}

\end{document}

%% file: Sections/SpecialNotation.tex
\usepackage{amsfonts, amsmath, amssymb}

\usepackage{enumitem} 

\usepackage{bm} 

\usepackage{mathrsfs} 

\usepackage{array} 

\usepackage{color} 

\newcommand{\argmin}{\operatornamewithlimits{argmin}}

\newcommand{\normone}[1]{\ensuremath{\,\!|\!| #1 | \! |_{1}}}
\newcommand{\normtwo}[1]{\ensuremath{\,\!|\!| #1 | \! |_{2}}}
\newcommand{\norm}[1]{\ensuremath{\,\!|\!| #1 | \! |}}

\newcommand{\inprod}[2]{\ensuremath{\langle #1 , \, #2 \rangle}}

\def\E{\mathbb{E}}
\def\R{\mathbb{R}}

\newcommand{\Rp}{\ensuremath{\R^p}}

\newcommand{\si}{\sum_{i=1}^n} 

\newcommand{\Hilbertspace}{\ensuremath{\mathcal H}}
\newcommand{\equationspace}{~~~~}

\newcommand{\data}{\ensuremath{X}}
\newcommand{\datav}{\ensuremath{{\bf x}}}

\newcommand{\datasi}{\ensuremath{{X^i}}}

\newcommand{\Matrix}{{\ensuremath{\mathcal{L}}}}

\renewcommand{\matrix}{\ensuremath{{\mathrm \Lambda}}}

\newcommand{\matrixtrue}{\ensuremath{{\matrix^*}}}

\newcommand{\matrixgeneric}{{\textcolor{black}{\ensuremath{\matrix'}}}}

\newcommand{\normtotal}{\textcolor{black}{\ensuremath{a}}}

\newcommand{\tr}{\ensuremath{\operatorname{tr}}}

\newcommand{\matrixest}{{\textcolor{black}{\ensuremath{{\widehat{\matrix}}}}}}

\newcommand{\measure}{{\textcolor{black}{\ensuremath{\nu}}}}

\newcommand{\tensorindexgeneric}[2]{_{#1,\dots,#2}}

\newcommand{\mypenalty}{\ensuremath{u}}
\newcommand{\dualpenalty}{\ensuremath{\tilde{\mypenalty}}}
\newcommand{\empiricalrisk}{\ensuremath{\widehat{d}}}
\newcommand{\populationrisk}{\ensuremath{d}}
\newcommand{\tuningparameter}{\ensuremath{r}}

\newcommand{\tensordimension}{\ensuremath{p}}
\newcommand{\regressionoutput}{\ensuremath{Y}}
\newcommand{\regressionoutputi}{\ensuremath{\regressionoutput^i}}
\newcommand{\regressioncovariates}{\ensuremath{{\bf z}}}
\newcommand{\regressioncovariatesi}{\ensuremath{\regressioncovariates^i}}
\newcommand{\regressioncovariatesit}{\ensuremath{(\regressioncovariates^i)^\top}}
\newcommand{\regressionnoise}{\ensuremath{E}}
\newcommand{\regressionnoisei}{\regressionnoise^i}

\newcommand{\lmoutput}{\ensuremath{y}}
\newcommand{\lmoutputi}{\ensuremath{\lmoutput^i}}
\newcommand{\lmcovariates}{\ensuremath{{\bf z}}}
\newcommand{\lmcovariatesi}{\ensuremath{\lmcovariates^i}}
\newcommand{\lmcovariatesit}{\ensuremath{(\lmcovariates^i)^\top}}
\newcommand{\lmnoise}{\ensuremath{\varepsilon}}
\newcommand{\lmnoisei}{\lmnoise^i}
\newcommand{\lmcoef}{{\beta}}
\newcommand{\lmcoeftrue}{\lmcoef^*}
\newcommand{\lmcoefest}{\widehat{\lmcoef}}
\newcommand{\lmerror}{\lmoutputi-\lmcovariatesi\lmcoef}

\newcommand{\tensorproduct}[1]{\ensuremath{\times_{#1}}}

\newcommand{\nmean}{\ensuremath{\frac{1}{n}\sum_{i=1}^{n}}}
\newcommand{\sumn}{\ensuremath{\sum_{i=1}^{n}}}
\newcommand{\tensorerror}{\ensuremath{\regressionoutputi-\matrixtrue \tensorproduct{1} \regressioncovariatesi}}

\newcommand{\tensorpredicta}{\ensuremath{\matrixtrue \tensorproduct{1} \regressioncovariatesi-\matrixest \tensorproduct{1} \regressioncovariatesi}}
\newcommand{\tensorpredictaformula}{\ensuremath{\matrixtrue \tensorproduct{1} \regressioncovariatesi-\matrix \tensorproduct{1} \regressioncovariatesi}}
\newcommand{\tensorerrorformula}{\ensuremath{\regressionoutputi-\matrix \tensorproduct{1} \regressioncovariatesi}}
\newcommand{\tensorentrysum}{\ensuremath{\sum_{i_1=1}^{\dimension_1}\dots \sum_{i_p=1}^{\dimension_p}}}
\newcommand{\tensorentryindex}{\ensuremath{i_1, \dots, i_p}}
\newcommand{\arraynorm}[1]{\ensuremath{\norm{#1}^2}}
\newcommand{\matrixdelta}{\ensuremath{( \matrixtrue-\matrixest)}}
\newcommand{\matrixdeltaformula}{\ensuremath{( \matrixtrue-\matrix)}}
\newcommand{\sums}[4]{\ensuremath{\sum_{#1}^{#2}\dots\sum_{#3}^{#4}}}
\newcommand{\sumsnew}[2]{\ensuremath{\sum_{#1_#2=1}^{\dimension_#2}\dots\sum_{#1_\tensordimension=1}^{\dimension_\tensordimension}}}
\newcommand{\Sigmainversequareroot}{\ensuremath{\Sigma^{-1/2}}}

\newcommand{\determinant}[1]{\ensuremath{\operatorname{det}#1}}
\newcommand{\density}[1]{\ensuremath{f_{#1}}}
\newcommand{\dimension}{\ensuremath{b}}
\newcommand{\oneton}[1]{\ensuremath{#1 \in \{1, \dots, n \}}}
\newcommand{\dataoneton}{\ensuremath{\data^1, \dots,\data^n}}

\newcommand{\abs}[1]{\ensuremath{|#1|}}

\newcommand{\gradient}{\ensuremath{\nabla(\populationrisk-\empiricalrisk)_\matrixest}}
\newcommand{\conditiona}{\ensuremath{\populationrisk(\matrix)-\empiricalrisk(\matrix)}}
\newcommand{\conditionaest}{\ensuremath{\populationrisk(\matrixest)-\empiricalrisk(\matrixest)}}
\newcommand{\mycondition}{\ensuremath{\log\density{\matrix}-\E_{\matrixtrue}\log\density{\matrix}}}
\newcommand{\spanMatrix}{\Hilbertspace}
\newcommand{\Wigradient}{\ensuremath{\nabla\mathcal{W}^i(\matrixest)}}
\newcommand{\Wi}{\ensuremath{\mathcal{W}^i}}
\newcommand{\element}[2]{\ensuremath{e^{(#1)}_{#2}}}
\newcommand{\elementinverse}[2]{\ensuremath{\sigma^{(#1)}_{#2}}}
\newcommand{\gmgradient}{\ensuremath{\sumn\big(\E_{\matrixtrue}T(\datasi)-T(\datasi)\big)}}

%% file: Sections/Intro.tex
\section{Introduction}
\label{sec: intro}
\subsection{Overview}
Maximum regularized likelihood estimators (MRLEs) are widely used in generalized linear regression, tensor response regression, and graphical modeling with high-dimensional data. It is thus of major interest to develop theory for this class of estimators. 

Our specific goal is a general finite sample theory for prediction. Existing results are typically  derived on a case-by-case basis. Moreover, many of these results also invoke restricted eigenvalues-type conditions~\citep[Section 6]{Buhlmann2011statistics}. Such conditions are not only stringent and unverifiable in practice but also unsuitable for prediction. For example, restricted eigenvalue conditions in regression  limit the correlations among the covariates.  However, although correlations can affect the identifiability of the parameters, for prediction, even perfectly collinear covariates do not necessarily have a negative impact; in contrast, collinearity can even be beneficial~\citep{Hebiri13,Dalalyan2017prediction}. We are thus interested in a theory that does not involve additional assumptions and provides bounds for a general class of MRLEs. Besides its abstract value, such a general theory can also provide support for specific examples of MRLEs, such as the recently introduced approaches to tensor regression~\citep{Zhou:2013ec,Li:2013tm,Sun:2016uj}, whose prediction properties have not been fully grasped. 

In this paper, we establish a general oracle inequality in terms of the Kullback-Leibler divergence. Oracle inequalities are a standard way to formulate finite sample bounds in high-dimensional statistics. The Kullback-Leibler divergence is a standard way to quantify prediction accuracies; it applies to any model and yet specializes to well-established and interpretable notions of prediction performance. Our proofs invoke only the convexity of the parametrization and the definiteness and positive homogeneity of the regularizers. This makes the result applicable to a variety of parametric and non-parametric models and allows for a broad class of convex and non-convex regularizers.

The remainder of this paper is organized as follows. We introduce the framework and the general result in Section~\ref{sec: general theory}. We then provide examples in Section~\ref{sec: examples}. We finally conclude with a brief discussion in Section~\ref{sec:discussion}. All proofs are deferred to the Appendices in the supplementary material: proofs for  the main result  to Appendix~A, proofs for the examples to Appendix~B, and proofs for the bounds of the empirical process terms in Appendix~C. In addition, Appendix~D contains notation and properties of tensors.

\subsection{Related Literature}\label{sectionLiteratureReview}
There are two types of oracle inequalities in the literature: so-called ``fast rate bounds'' and so-called  ``slow rate bounds.''  ``Fast rate bounds'' are proportional to the square of the regularization parameter. Many representatives of this type of bounds are found in the literature, such as \citet{Bunea2007sparsity,Raskutti2015convex} for  regression, \citet{Ravikumar2011high} for graphical models, and more generally, \citet{Buhlmann2011statistics,Van2016estimation} and references therein. For example, the corresponding bounds for  lasso prediction are of the form $s\log p/(w^2 n),$ where $s$ is the number of non-zero elements in the true regression vector, $p$ is the number of parameters, $w$ is the restricted eigenvalue, and $n$ is the number of observations. These bounds are typically considered fast,  because they can match minimax rates, see~\cite{Verzelen12} and references therein. However, they rely on sparsity, and more importantly, they invoke restricted eigenvalue-type conditions or concern computationally challenging estimators instead \citep{BTW07a,DT07,RigTsy11,DT12a,DT12b}. Moreover, these eigenvalue-type assumptions are unverifiable and often unrealistic in practice, and even if they hold, the additional factors (such as  $s$ and $1/w^2$ for lasso) can be large. 


On the other hand, oracle inequalities for prediction have be derived without sparsity or restricted eigenvalue conditions for  lasso-type estimators \citep{greenshtein2004persistence,RigTsy11,MM11,Koltchinskii2011nuclear,Huang:2012ww,chatterjee2013assumptionless,buhlmann2013statistical,chatterjee2014new,Lederer2016oracle,Dalalyan2017prediction}. For example, the corresponding  bounds for lasso prediction are of the form  $\sqrt{\log p/n}\normone{\lmcoeftrue}$, where  $\lmcoeftrue$ is the true regression vector. Such bounds are typically referred to as ``slow rate bounds,'' because on a high level, the rates are $1/\sqrt n$ rather than  $1/n.$ However, there are no questionable assumptions involved, and for regression, it has even been shown  that $1/\sqrt n$ is the optimal rate in the absence of further assumptions~\citep{Foygel2011fast,Zhang2015optimal,Dalalyan2017prediction}. Overall, this means that ``fast rate bounds'' are not necessarily fast and ``slow rate bounds'' are not necessarily slow. To correct the misleading nomenclature, \citet{Lederer2016oracle} suggested replacing the term ``fast rate bound'' with ``sparsity bound' and ``slow rate bound'' with ``penalty bound.'' 

Although some examples of MRLEs have been equipped with assumptionless bounds, many other examples still lack such guarantees (or any prediction guarantees altogether). More generally, a broadly applicable prediction theory for MRLEs is still in need. 

\subsection{Our Contribution}
The contribution of this work is two-fold. First, Theorem~\ref{MAINTHEORY} provides a general prediction guarantee for MRLEs in terms of the Kullback-Leibler loss. Besides  being the first assumptionless bound in such a broad setting, the result  specializes correctly to  known results, such as for lasso, where the corresponding rates have been shown to be optimal up to log-factors. Second, we show that applications of the general theorem to specific examples lead to new guarantees in tensor response regression, generalized linear tensor regression, and graphical modeling. The theory thus also establishes new insights into individual cases of MRLEs.

%% file: Sections/GeneralTheory.tex

\section{General Theory}
\label{sec: general theory}
In this section, we present the general theory comprising  the model classes, estimators, and the main result. The theory applies to an extremely wide range of data and methods; we discuss many important examples in Section~\ref{sec: examples}. As for the models, we consider random vectors $\data \in \mathcal{X}$ in a non-empty set~$\mathcal{X}$ distributed according to a density $f \in \mathcal{F}$ in a general class~$\mathcal{F}.$ We assume that the densities in $\mathcal{F}$ can be parametrized as $\density{\matrix}$ with parameter $\matrix\in\Matrix$ that belongs to  a convex, non-empty set~$\Matrix$ in a real Hilbert space~\Hilbertspace\ and  $\log\density{\matrix}-\E_{\matrixgeneric}\log\density{\matrix}$ is convex in $\matrix$ for fixed $\matrix'\in\Matrix.$ A classical example for this setup is the case of exponential families in the natural form~\citep[Lemma 2.1]{Berk1972consistency}, see also~\citet{Johansen1979introduction,Brown1986fundamentals}. In general, however, the parametrization can be arbitrary as long as the convexity condition is fulfilled, and the parameter space can well be infinite-dimensional. In view of this very general framework, with the convexity of the parametrization being the only requirement on the models, the following theory applies to a large class of models.

The targets of our study are MRLEs in the described setup. Maximum likelihood estimation is one of the most widely accepted approaches to understand data, and regularization is a standard technique to incorporate additional structure or information. A contemporary playground for MRLEs is high-dimensional statistics, where a tremendous amount of research  centers around regularization based on sparsity structures~\citep{Buhlmann2011statistics,Giraud2014introduction,Hastie2015statistical}.  Given data $\data$, we consider MRLEs of the form (assumed to exist)
\begin{equation}\label{eqn:MRLE}
\widehat{\matrix}  \in  \argmin_{ \matrix \in \Matrix}\big\{ -\log f_\matrix(\data) + \tuningparameter \mypenalty (\matrix) \big\},
\end{equation}
where $\tuningparameter > 0$ is a regularization parameter and $\mypenalty : \Hilbertspace \mapsto [0, \infty]$ is a regularization with properties
\begin{align}
&\mypenalty(\matrix) = 0~~\Leftrightarrow~~\matrix  =  0 ,\label{assume:only0}\\
& \mypenalty(t\matrix) = t\mypenalty(\matrix)\equationspace \forall \matrix \neq 0, t \geq 0.\label{assume: positive homogeneous}
\end{align}
These two properties allow us to formulate dual functions that generalize the classical notion of dual  norms and the corresponding H\"older-like inequalities, see the definition of $\dualpenalty$ below and Lemma~A.1  in Appendix~A. Indeed, one can check readily that the properties are met by norms, including the weighted norm penalities considered in~\citet{Zou:2006du,Van2008high,Gramfort2012mixed,Bu:2017tn} and others. However, the properties are also satisfied by the more general concept of gauges, which requires convexity in addition to \eqref{assume:only0} and \eqref{assume: positive homogeneous}, and which has become an increasingly popular  subject of optimization theory~\citep{Friedlander:2016ge,Aravkin:2017ur}. Furthermore,  we allow for non-convex functions: for example, the category of regularizers covers  $\ell_q$-operators, $\ell_q(\matrix):=(\sum_{j=1}^p|\matrix_j|^q)^{1/q}$ for $\matrix\in\Rp,$  even in the non-convex case $q \in (0,1)$; we refer to~\citet{Foucart2009sparsest} for corresponding optimization techniques. More generally, it covers Minkowski functionals $\mypenalty(\matrix):=\inf\{a>0:\matrix\in a\mathcal K\}$ with  level set $\mathcal K$ that is bounded and contains an open set around the origin, but is potentially non-symmetric and non-convex. Altogether, we consider a very general class of estimators.

A standard measure to assess the accuracy of estimators is the Kullback-Leibler divergence~\citep{Huntsberger1981elements}. This measure is particularly suited for our theory, because it can be formulated independently of the  model class at hand and yet specifies to established  measures in  applications. For given $\matrix, \matrixgeneric \in \Matrix$, the Kullback-Leibler divergence from $f_\matrix$ to $f_\matrixgeneric$ is defined as
\begin{equation*}
\populationrisk(\matrix;\matrixgeneric) :  =  \E_{\matrixgeneric}\text{log}\Big(\frac{f_{\matrixgeneric}(\data)}{f_\matrix(\data)}\Big)  .
\end{equation*}
Given data $\data,$ the empirical version of $\populationrisk(\matrix;\matrixgeneric)$ is then
\begin{equation}\label{eq: empirical risk}
\empiricalrisk(\matrix;\matrixgeneric \mid \data) : =   \text{log}\Big(\frac{f_{\matrixgeneric}(\data)}{f_\matrix(\data)}\Big)  .
\end{equation}
For ease of notation, we assume in the following $X\sim f_\matrixtrue$ for the ``true'' parameter $\matrixtrue\in\Matrix$ and set  $\populationrisk(\matrix):=\populationrisk(\matrix; \matrixtrue)$ and $\empiricalrisk(\matrix):=\empiricalrisk(\matrix;\matrixtrue \mid \data)$.

We can now formulate an oracle inequality for the MRLE given in~\eqref{eqn:MRLE}. For this, the function~$\dualpenalty$ at $\matrix\in\Hilbertspace$ is defined as the dual of $\mypenalty$ by
\begin{equation*}
\dualpenalty(\matrix)  := \sup\big\{  \inprod{\matrix}{\matrix'}  \mid \matrix' \in \spanMatrix, \mypenalty(\matrix')  \leq  1 \big\},
\end{equation*}
where $\inprod{\cdot}{\cdot}$ is the inner product on \Hilbertspace. Moreover, $\gradient\in\Hilbertspace$ denotes any subgradient of $\populationrisk(\matrix)-\empiricalrisk(\matrix)$ at $\widehat{\matrix}$. We then find the following.
\begin{theorem}[oracle inequality]\label{MAINTHEORY} 
	For all $\tuningparameter \geq \dualpenalty\big(\gradient\big)$, it holds that 
	\begin{equation*}
		\populationrisk(\widehat{\matrix}) \leq \tuningparameter \mypenalty (\matrixtrue)+\tuningparameter \mypenalty (-\matrixtrue).
	\end{equation*}
\end{theorem} 
\noindent The bound has three building blocks. First, the Kullback-Leibler loss is used as a measure of the accuracy of the MRLEs. In many examples, this loss equals a classical prediction loss. Second, the ``noise term'' $\dualpenalty\big(\gradient\big)$ forms a lower bound on the regularization parameter. This term can typically be controlled by using bounds from empirical process theory. Finally, the (symmetrized) size of the true model $\mypenalty (\matrixtrue)+\mypenalty (-\matrixtrue)$ scales the accuracy bounds. The size is measured in terms of \mypenalty, which reflects the rationale for choosing \mypenalty\ in the first place.

Theorem~\ref{MAINTHEORY} is in the form of  an oracle inequality, which is a standard way to capture the performance of regularized estimators~\citep{Buhlmann2011statistics}. Importantly, oracle inequalities provide finite sample guarantees and are thus, as opposed to asymptotic results, of direct relevance in practice. However, the inequality also entails upper bounds on the rates of convergence. Note first that the size of the true model can be considered as a basically constant  factor; indeed, in view of the motivation of regularization being that there is a true model with reasonable  size in $\mypenalty$, largely inflating  values of $\mypenalty(\pm\matrixtrue)$ would indicate an inappropriate choice of the regularization function. As a conclusion, one can derive bounds for the rates of convergence essentially  by looking at the regularization parameter~\tuningparameter.

An immediate question is whether the bounds in Theorem~\ref{MAINTHEORY} are optimal. To answer this question, we first recall  that for specific examples that fit our general framework, ``fast rate'' bounds proportional to $\tuningparameter^2$ rather than $\tuningparameter$ have been derived, but despite the inaccurate nomenclature, their rates are not necessarily fast. In particular, bounds proportional to $\tuningparameter^2$ contain additional factors that can slow down the rates, and more directly for scalable estimators, the known bounds rely on strong additional assumptions. Instead, it has been shown that bounds proportional to $\tuningparameter$ are optimal in lasso-type regression in the absence of further assumptions, which means that the bounds in Theorem~\ref{MAINTHEORY} are  indeed optimal in the sense that they cannot be improved in general --- see Sections~\ref{sectionLiteratureReview} and~\ref{sec:examplesf} for details. 

In summary, Theorem~\ref{MAINTHEORY} provides bounds for a wide range of models and corresponding MRLEs. Therefore, the theorem is an umbrella for bounds linear in $\tuningparameter$ that have been derived for specific examples previously. The proof, however, differs from the previous ones in the way that it uses convexity arguments, H\"older-type inequalities, and connections between the log-likelihood and the Kullback-Leibler loss. Furthermore, and more importantly,  Theorem~\ref{MAINTHEORY} also entails guarantees for models and estimators that have not yet been equipped with assumptionless bounds - or any bounds at all.

%% file: Sections/Examples.tex
\section{Examples}
\label{sec: examples}
We now give explicit bounds for  high-dimensional tensor response regression, generalized linear tensor regression, and graphical models. The bounds are the first ones to provide assumptionless Kullback-Leibler guarantees for MRLEs in these models. An exception is linear regression with lasso-type regularization, where assumptionless guarantees have been derived before. We show that we recover the known bounds in this case. 

\subsection{Tensor Response Regression}\label{sec:examplesf}
Our first example is tensor response regression. In a standard notation (see~\citet{Kolda2006multilinear} or our Appendix~D for details), tensor response regression is based on models of the form
\begin{equation*}
\label{model: tensor regression}
\regressionoutputi= \matrixtrue \tensorproduct{1}\regressioncovariatesi+\regressionnoisei \equationspace \big(\oneton{i}\big),
\end{equation*}
where $\regressionoutputi\in\R^{\dimension_2\times\dots\times \dimension_\tensordimension}$ is a  $(p-1)$th order tensor response, $\matrixtrue \in \Matrix \subset \R^{\dimension_1\times \dots \times \dimension_\tensordimension}$ is a $p$th order tensor coefficient, $\regressioncovariatesi\in\R^{1\times \dimension_1}$ is a fixed or random row-vector of covariates, and $\regressionnoisei \in \R^{\dimension_2\times\dots\times \dimension_\tensordimension}$ is random $(p-1)$th  order tensor noise. The operation $\matrixtrue \tensorproduct{1} \regressioncovariatesi$ denotes the mode-$1$ product of $\matrixtrue$ and $\regressioncovariatesi$. 

Our goal is to estimate the predictive structure of the above model. Assuming that the noise tensors $\regressionnoisei$ are mutually independent, the MRLEs in~\eqref{eqn:MRLE} are of the form
\begin{equation*}
\matrixest\in\argmin_{ \matrix \in\Matrix}\big\{ -\si \log\density{\matrix}(\regressionoutputi\mid\regressioncovariatesi)+\tuningparameter\mypenalty(\matrix) \big\}.
\end{equation*}
For computational ease, $\Matrix$ is usually chosen as a set of low-rank tensors~\citep{Rabusseau:2016um,Sun:2016uj}. In any case, if the conditional density $\density{\matrix}(\regressionoutputi\mid\regressioncovariatesi)$ is parametrized such that $\matrix\mapsto\mycondition$ is convex,  we can derive statistical guarantees in conditional Kullback-Leibler loss from Theorem~\ref{MAINTHEORY}. Importantly,  we do not impose any additional restriction on the covariates or the noise; for example, we allow the covariates to be correlated with the noise. Typical regularizers for third order tensors, for example, include the sparsity-inducing regularizer at the entry level $\mypenalty(\matrix):= \sums{i_1=1}{\dimension_1}{i_3=1}{\dimension_3}\abs{\matrix_{i_1i_2i_3}}$, at the fiber level $\mypenalty(\matrix):= \sum_{i_2=1}^{\dimension_2}\sum_{i_3=1}^{\dimension_3}\normtwo{\matrix_{\cdot i_2i_3}}$, and at the slice level $\mypenalty(\matrix):= \sum_{i_3=1}^{\dimension_3}\norm{\matrix_{\cdot \cdot i_3}}_F:=\sum_{i_3=1}^{\dimension_3}\sqrt{\sum_{i_1=1}^{b_1}\sum_{i_2=1}^{b_2}\abs{\matrix_{i_1i_2i_3}}^2}$ and the low-rank inducing regularizer $\mypenalty(\matrix) := \norm{\matrix}_*$ with $\norm{\cdot}_*$ the tensor nuclear norm defined in~\citet{Raskutti2015convex}. Our framework covers all these examples.

For illustration, we consider tensor response regression with zero-mean array normal noise~\citep{Akdemir2011array,Hoff2011separable}, the most widely-used representative of the above model class. The conditional Lebesgue density of $\regressionoutputi$ given by~\cite{Hoff2011separable} is
\begin{equation*}
	\density{\matrix}(\regressionoutputi\mid \regressioncovariatesi)=(2\pi)^{-b/2}\big( \prod_{k=2}^{p}|\Sigma_k|^{-b/(2b_k)} \big)\cdot \exp\big(-\frac{1}{2}\norm{(\tensorerrorformula)\times \Sigma^{-1/2}}^2\big),
      \end{equation*}
where $\norm{\cdot}^2$ is the array norm, $b:=\prod_{j=2}^p{b_j}$, $\Sigma_k=A_kA_k^\top\in\R^{\dimension_k\times \dimension_k}$ with non-singular real matrix $A_k \in \R^{\dimension_k\times \dimension_k}$, $\Sigma^{-1/2}=\{A_2^{-1},\dots,A_\tensordimension^{-1}\}$, and $\times$~denotes the tensor product.
One can check readily that $\log \density{\matrix}(\regressionoutputi\mid\regressioncovariatesi) -\E_\matrixtrue\log \density{\matrix}(\regressionoutputi\mid\regressioncovariatesi)$ is linear in $\matrix$. Hence our theory applies; in particular, Theorem~\ref{MAINTHEORY} specializes to array normal models as follows.
\begin{lemma}[tensor response regression with array normal noise]
	\label{lemma: tensor regression with array normal noise} ~~For all $\tuningparameter \geq \dualpenalty\big(\sum_{i=1}^{n}\big( \regressionnoisei \times \Sigma^{-1}\tensorproduct{1} \regressioncovariatesit  \big)\big)$, where $\Sigma^{-1}=\{\Sigma_2^{-1},\dots,\Sigma_\tensordimension^{-1}\}$, it holds that
	\begin{equation*}
		\populationrisk(\matrixest) \leq \tuningparameter \mypenalty(\matrixtrue)+\tuningparameter \mypenalty(-\matrixtrue),
	\end{equation*}
with Kullback-Leibler loss
	\begin{equation*}
		\populationrisk(\matrixest) =\frac{1}{2}\sum_{i=1}^{n}\norm{( \matrixtrue-\matrixest)\tensorproduct{1}\regressioncovariatesi\times \Sigma^{-1/2}}^2.
	\end{equation*}
\end{lemma}
\noindent This bound entails that MRLEs for tensor response regression with array normal noise are consistent in average conditional Kullback-Leiber loss under minimal assumptions. Our results thus complement the known consistency guarantees, which hold for specific tensor regressions with additional constraints on the covariates~\citep{Raskutti2015convex,Sun:2016uj}. In addition, Lemma~\ref{lemma: tensor regression with array normal noise} elucidates the interpretation of the conditional Kullback-Leibler loss as a prediction loss. 

For an instantiation of the bound, assume $\sum_{i=1}^n(\regressioncovariatesi)_j^2/n=1,$ $j\in\{1,\dots,b_1\},$ and $(\Sigma^{-1}_{k})_{i_k i_k}=h_k^2,$ $h_k>0,$ $k\in\{2,\dots,p\},$ $i_k\in\{1,\dots,b_k\}.$ Consider the sparsity-inducing regularizer at the entry level $\mypenalty(\matrix) := \sums{i_1=1}{b_1}{i_p=1}{b_p}\abs{\matrix_{\tensorindexgeneric{i_1}{i_p}}}.$ Then,  the tuning parameter~$r$ can be calibrated such that with probability at least $1-2\exp(-t^2),$ it holds that
\begin{equation*}
  \populationrisk(\matrixest) \leq 2\big(\prod_{k=2}^{p}h_k\big)\,\sqrt{2n\big(t^2+\log(\prod_{j=1}^{p}b_j)\big)}\,\sums{i_1=1}{b_1}{i_p=1}{b_p}\abs{\matrix^*_{\tensorindexgeneric{i_1}{i_p}}}.
\end{equation*}
This concrete bound follows from Lemma~C1. That lemma and its proof  --- as well as all other technical derivations --- are deferred to the Appendix.

While the above results for tensor regression are novel, assumptionless bounds  for simple regression with lasso-type estimators such as lasso~\citep{Tibshirani1996regression}, group lasso~\citep{Yuan:2006jy}, sparse group lasso~\citep{Simon2013sparse}, and slope estimator~\citep{Bogdan2015slope} have been derived before. Simple linear regression is thus an ideal test case to confirm that our results specialize correctly.  To this end, we first observe that for $p=2$ and $\dimension_2=1$, tensor response regression with array normal noise reduces to  ordinary linear regression of the form
\begin{equation*}
\lmoutputi=\lmcovariatesi \lmcoeftrue + \lmnoisei \equationspace \big(\oneton{i}\big),
\end{equation*}
where $\lmoutputi \in \R$ is a scalar response,  $\lmcovariatesi \in \R^{1\times \dimension_1}$ is a row-vector of covariates,  $\lmcoeftrue \in \R^{\dimension_1}$ is the regression vector, and $\lmnoisei \in \R$ is noise distributed as $\mathcal{N}(0, \sigma^2)$. For $\mypenalty(\lmcoef):=\normone{\lmcoef}:= \sum_{i=1}^{b_1}\abs{\lmcoef_i}$, the MRLE~\eqref{eqn:MRLE} becomes 
 \begin{equation*}
	\lmcoefest\ \in\ \argmin_{\lmcoef \in \R^{b_1}}\big\{\frac{1}{2\sigma^2}\sum_{i=1}^{n}\big(\lmerror\big)^2 + \tuningparameter  \norm{\lmcoef}_1\big\},
\end{equation*}
which, setting~$\tuningparameter =\tuningparameter'/(2\sigma^2)$, can be written in the standard lasso-form
\begin{equation*}
\label{eq: lasso}
	\lmcoefest\ \in\ \argmin_{\lmcoef \in \R^{b_1}}\big\{\sum_{i=1}^{n}\big(\lmerror\big)^2 + \tuningparameter' \norm{\lmcoef}_1\big\}.
\end{equation*}
If $\tuningparameter' \geq 2\norm{\sum_{i=1}^{n}\lmcovariatesit\lmnoisei}_\infty$, Lemma~\ref{lemma: tensor regression with array normal noise} now implies the bound
\begin{equation*}
	\sum_{i=1}^{n}\big(\regressioncovariatesi \beta^*- \regressioncovariatesi\widehat{\beta}\big)^2 \leq 2\tuningparameter' \norm{\beta^*}_1.
\end{equation*}
This result equals the classical penalty bound for lasso prediction, see~\citet[Equation~3]{Hebiri13} for example. It has been shown that these bounds are essentially optimal in the absence of other assumptions~\citep{Foygel2011fast,Zhang2015optimal,Dalalyan2017prediction}. Along the same lines, one can also show that our bounds specify correctly for the other lasso-type estimators mentioned above~\citep[Section 3]{Lederer2016oracle}, and similarly, for trace regression~\citep[Theorem~1]{Koltchinskii2011nuclear}.

\subsection{Generalized Linear Tensor Regression}
Our second example is generalized linear tensor regression. The corresponding models consist of two components~\citep{Zhou:2013ec,Li:2013tm}: an exponential family  distribution and a link function. The exponential family distribution reads
\begin{equation*}
f(\lmoutputi \mid \theta^i) =\exp\big( \frac{\lmoutputi\theta^i -b(\theta^i)}{\alpha}+c(\lmoutputi,\alpha) \big) \equationspace \big(\oneton{i}\big),
\end{equation*}
where $\lmoutputi \in \R$ is a scalar response, $\theta^i \in \R$ is the natural parameter, $\alpha >0 $ is the overdispersion factor, and $b, c$ are known real-valued functions. The  link function $g: \R \mapsto \R,$ assumed strictly increasing, provides a linear connection between the mean functions $\E(\lmoutputi\mid \theta^i)$ and tensor predictors $\lmcovariatesi \in \R^{\dimension_1\times \dots \times \dimension_\tensordimension}$  according to  
\begin{equation*}
	g\big(\E(\lmoutputi\mid \theta^i)\big) = \inprod{\matrixtrue}{\lmcovariatesi},
\end{equation*}
where $\matrixtrue \in \Matrix \subset \R^{\dimension_1\times \dots \times \dimension_\tensordimension}$ and $\inprod{\cdot}{\cdot}$ is the tensor inner product. One can check that $b'(\theta^i)=\E(\lmoutputi\mid \theta^i)$. With canonical link $g := (b')^{-1}$, it holds that $\theta^i = \inprod{\matrixtrue}{\lmcovariatesi}$ and the distribution of the response~$\lmoutputi$ conditioned on $\lmcovariatesi$ has density
\begin{equation*}
f_{\matrixtrue}(\lmoutputi\mid \lmcovariatesi) =\exp\big( \frac{\lmoutputi \inprod{\matrixtrue}{\lmcovariatesi} -b(\inprod{\matrixtrue}{\lmcovariatesi})}{\alpha}+c(\lmoutputi,\alpha) \big).
\end{equation*}
Further, by introducing basis functions in the mean models, it is straightforward to extend this parametric setting to non-parametric frameworks. In sum, generalized linear tensor regression provides  very flexible model classes for scalar responses.

We can now turn to the corresponding MRLEs. Given $n$ independent observations $(\lmoutputi, \lmcovariatesi)$ and considering the canonical link, the MRLEs  in~\eqref{eqn:MRLE} become
\begin{equation*}
\matrixest \in \argmin_{ \matrix \in \Matrix}\big\{ \frac{1}{\alpha}\sumn\big( -\lmoutputi\inprod{\matrix}{\lmcovariatesi}  +b(\inprod{\matrix}{\lmcovariatesi})\big)+\tuningparameter\mypenalty(\matrix)  \big\}.
\end{equation*}
Similarly to tensor response regression, optimization over the full set~$\R^{\dimension_1\times \dots \times \dimension_\tensordimension}$ is computationally challenging due to high-dimensionality of the problem. Thus, \Matrix\ is typically chosen considerably smaller, with the belief that the true parameter has some additional structure. For example, a choice proposed in~\citet{Zhou:2013ec} is $\Matrix := \{ \matrix \in \R^{\dimension_1\times \dots \times \dimension_\tensordimension} \mid \matrix =\sum_{i=1}^{m}\beta_1^{(i)}\circ\dots \circ\beta_\tensordimension^{(i)} \}$, where $m$ is a fixed integer, $\beta_j^{(i)}\in \R^{\dimension_j}$,  and $\circ$ denotes the  outer product. 

Let us now apply Theorem~\ref{MAINTHEORY} to equip MRLEs in generalized linear tensor regression with theoretical guarantees. For this, note that the log-parametrization here is again linear, so that the main theorem indeed applies and yields the following results.
\begin{lemma}[generalized linear tensor regression with canonical link]
	\label{lemma: glm}
	For all $\tuningparameter\geq \dualpenalty\big(\frac{1}{\alpha}\sum_{i=1}^{n}\big(\lmoutputi-\E_{\matrixtrue}(\lmoutputi)\big) \lmcovariatesi \big)$, it holds that	
	\begin{equation*}
		\populationrisk(\matrixest) \leq \tuningparameter \mypenalty(\matrixtrue) + \tuningparameter\mypenalty(-\matrixtrue),
	\end{equation*}
with Kullback-Leibler loss
	\begin{equation*}
		\populationrisk(\matrixest)= \frac{1}{\alpha}\sumn\big( g^{-1}(\inprod{\matrixtrue}{\lmcovariatesi})\cdot\inprod{\matrixtrue-\matrixest}{\lmcovariatesi} - b(\inprod{\matrixtrue}{\lmcovariatesi}) + b(\inprod{\matrixest}{\lmcovariatesi}) \big),
	\end{equation*}
and $\E_{\matrixtrue}(\lmoutputi)$ denotes the conditional expectation of $\lmoutputi$ on $\lmcovariatesi$ here.
\end{lemma}
\noindent To the best of our knowledge, this is the first oracle inequality for regularized generalized linear tensor regression. 

As a special case, Lemma~\ref{lemma: glm} applies to ordinary logistic regression, where $\tensordimension=1$, the canonical link is $g(x) = \log\big(x/(1-x)\big)$, and the MRLEs with a general regularizer are in the form of
\begin{equation*}
\matrixest \in \argmin_{ \matrix \in \R^{b_1}}\big\{\sumn\big( -\lmoutputi\inprod{\matrix}{\lmcovariatesi}  +\log\big(1+e^{\inprod{\matrix}{\lmcovariatesi}}\big)\big)+\tuningparameter\mypenalty(\matrix)  \big\}.
\end{equation*}
Lemma~\ref{lemma: glm} then implies the bound
\begin{equation*}
d(\matrixest) \leq \tuningparameter \mypenalty(\matrixtrue) + \tuningparameter\mypenalty(-\matrixtrue)
\end{equation*}
for $\tuningparameter \geq \dualpenalty\big(\si\big(\lmoutputi-e^{\inprod{\matrixtrue}{\lmcovariatesi}}/\big(1+e^{\inprod{\matrixtrue}{\lmcovariatesi}}\big)\big)\lmcovariatesi\big)$. 
For $\ell_1$-regularization, the tuning parameter~$r$ can be calibrated such that with probability at least $1-2\exp(-t^2),$ it holds that
\begin{equation*}
\populationrisk(\matrixest) \leq 2\sqrt{\frac{1+2\max_i\{p^i(1-p^i)\}}{3}\, n(t^2+\log b_1)}\,\sum_{j=1}^{b_1}\abs{\matrix^*_j},
\end{equation*}
where $p^i:=\E_\matrixtrue(\lmoutputi)$. 
We refer to Lemma~C2 in Appendix~C for details.
The bound complements results for (weighted) $\ell_1$-regularized logistic regression that have been derived under additional assumptions, see~\citet{Van2008high} and \citet[Chapter 12.4]{Van2016estimation}.

\subsection{Graphical Models}
Our third example is graphical modeling. We consider the exponential trace framework~\citep{Lederer2016graphical}, which encompasses standard types of graphical models, such as Gaussian graphical models~\citep{Yuan2007model,Friedman2008sparse}, non-paranormal graphical models~\citep{Gu2015local}, and Ising models~\citep{Lenz20,Brush67}; we refer to \citet[Section 2]{Lederer2016graphical} for details.

Exponential trace models are  based on densities of the form 
\begin{equation*}
f_\matrix(\datav)=\exp\big(-\inprod{\matrix}{T(\datav)}-\normtotal(\matrix)\big)\equationspace\big(\matrix\in\Matrix\big)
\end{equation*}
with respect to some $\sigma$-finite measure $\measure$ on $\R^p$. Here, the matrix-valued parameter~$\matrix$ encodes the dependence structure of the random vector $X \in\mathcal X\subset \R^p$, the matrix-valued function $T(\cdot)$ on $\mathcal X$ determines  how the data enters the model, $\normtotal(\matrix)$ is the normalization, and  $\Matrix$ is a convex set of matrices~\matrix\ with  finite normalization. Given independent observations $\dataoneton$ of $\data$, the MRLEs in~\eqref{eqn:MRLE} are of the form
\begin{equation*}
	\matrixest \in \argmin_{ \matrix \in \Matrix}\big\{\si \inprod{\matrix}{T(\data^i)}+n\hspace{0.1mm}\normtotal(\matrix) +\tuningparameter\mypenalty(\matrix) \big\}.
\end{equation*}
Because the function $\mycondition$ is linear in $\matrix$, we can then apply Theorem~\ref{MAINTHEORY} to derive the following bound.
\begin{lemma}[graphical models]\label{cor: specific slow rate bound}
	For all $\tuningparameter  \geq   \dualpenalty\big(\si\big(\E_{\matrix^*}T(\data^i)-T(\data^i)\big)\big)$, it holds that
	\begin{equation*}
	\populationrisk(\widehat{\matrix}) \leq \tuningparameter \mypenalty (\matrix^*)+\tuningparameter \mypenalty (-\matrix^*) ,
	\end{equation*}
	with Kullback-Leibler loss
	\begin{equation*}
	\populationrisk(\matrixest)  =  \inprod{\sumn\E_{\matrixtrue}T(\datasi)}{\matrixest-\matrixtrue}-n\normtotal(\matrixtrue) +n\normtotal(\matrixest). 
	\end{equation*}
\end{lemma}
\noindent  The Kullback-Leibler loss is a standard predictive risk for graphical models~\citep{Yuan2007model,Shevlyakova2012KLD}  and has a geometric interpretation as the difference between $\normtotal(\matrixest)$ and the tangent approximation of $\normtotal(\matrixest)$ at $\matrixtrue$~\citep[Chapter 5.2.2]{Wainwright08}.

As a special case, Lemma~\ref{cor: specific slow rate bound} applies to the graphical lasso for multivariate Gaussian data. Recall that the graphical lasso~\citep{Yuan2007model, Friedman2008sparse} is formulated as 
\begin{equation*}
\matrixest \in\argmin_{ \matrix \in S_{++}^{p}}\big\{\nmean\tr\big(\data^i(\data^i)^\top\matrix\big)-\log\det \matrix +\tuningparameter'\normone{\mbox{vec}(\matrix)} \big\},
\end{equation*}
where $S_{++}^{p} $ is the set of positive definite $p\times p$ matrices and  $\data^1, \dots, \data^n$ are i.i.d. samples from a centered Gaussian distribution with unknown covariance matrix~$(\matrixtrue)^{-1}$. The Kullback-Leibler loss of $\matrixest$ reads
\begin{equation*}
\frac{1}{n}\populationrisk(\matrixest)=\frac{1}{2}\big( \inprod{\matrixest}{(\matrixtrue)^{-1}} - \log\determinant{\matrixest} + \log\determinant{\matrixtrue} -p\big),
\end{equation*}
which is equivalent (up to the factor 1/2) to Stein's loss of the centered multivariate Gaussian distribution with covariance matrix $(\matrixest)^{-1}$~\citep{James1961estimation}. Thus, if $\tuningparameter' \geq  \norm{ \mbox{vec}\big((\matrixtrue)^{-1} - \frac{1}{n}\sum_{i=1}^{n}\data^i(\data^i)^\top\big)}_\infty$,  Lemma~\ref{cor: specific slow rate bound} yields 
\begin{equation*}
\inprod{\matrixest}{(\matrixtrue)^{-1}} - \log\determinant{\matrixest} + \log\determinant{\matrixtrue} -p \leq 2\tuningparameter'\normone{\mbox{vec}(\matrixtrue)}  .
\end{equation*} 
Following Lemma~C3 in Appendix~C, the tuning parameter~$r'$ can be calibrated such that with probability at least $1-4\exp(-t^2),$ it holds that
\begin{equation*}
\frac{1}{n}\populationrisk(\matrixest) \leq 160 \max_{k\in\{1,\dots, p\}}\big((\matrixtrue)^{-1}\big)_{kk}\sqrt{\frac{1}{n}\big(t^2 + \log\big(p(p-1)\big)\big)}\,\sum_{i_1=1}^{p}\sum_{i_2=1}^{p}\abs{\matrix^*_{i_1i_2}}
\end{equation*}
for all $t$ such that $0 <t<\sqrt{n/4-\log\big(p(p-1)\big)}$. 
Our theory thus establishes the rate~$\sqrt{\log p/n}$ for the graphical lasso in the Kullback-Leibler loss. This prediction rate complements the known estimation rates  $\sqrt{\log p/n}$ in vectorized-matrix $\ell_\infty$-norm and $\sqrt{\min\{s+p, d^2\}\log p/n}$ in spectral norm~\citep{Ravikumar2011high}, where~$s$ is the number of non-zero elements in~$\matrixtrue$ and $d$ is the maximum node degree. However, those results additionally require the mutual incoherence condition. Our prediction rate also complements the known estimation rate $\sqrt{(s+p)\log p/n}$ in Frobenius norm and spectral norm~\citep{Rothman2008sparse}, which requires only mild assumptions on the population covariance matrix.


%% file: Sections/Discussion.tex

\section{Discussion}\label{sec:discussion}
We have established assumptionless oracle inequalities for a general class of maximum regularized likelihood estimators. For regression, the inequalities match known lower bounds up to log-factors. We conjecture that the same is true more generally; in particular, we believe that general counter-examples to ``fast rates'' can be generated similarly as in the regression case.

%% file: Sections/Acknowledgment.tex
\section*{Acknowledgment}
We thank  Mohamed Hebiri and Jon Wellner for the many inspiring discussions and insightful suggestions. We also thank Jacob Bien, Roy Han, Joseph Salmon, Noah Simon, and  Yizhe Zhu for valuable input.

%% file: Sections/MainProofs.tex

\section{Proof of Theorem~1}
\label{appendix: main}
Before proving Theorem~1, we first introduce a lemma about the regularizer. Throughout, we use the convention $0\cdot\infty:=\infty.$

\begin{lemma}[inner product inequality]\label{lemma: innerProduct} 
	Let $\matrix, \matrixgeneric \in  \spanMatrix$. 
	It holds that
	\begin{equation}\label{eq:innerProductInequal}
	\inprod{\matrix}{\matrixgeneric} \leq \dualpenalty (\matrix)\mypenalty (\matrixgeneric) .
	\end{equation}
\end{lemma}

\begin{proof}[of Lemma~\normalfont{A1}] The proof consists of two steps. First, we show that Inequality (\ref{eq:innerProductInequal}) holds in the case $\mypenalty(\matrixgeneric) = 0  $. Second, we show that Inequality (\ref{eq:innerProductInequal}) holds in the case $\mypenalty(\matrixgeneric)\neq 0 $. In view of the mentioned convention, we can assume that $\mypenalty(\matrixgeneric)<\infty.$\\
	
	{\it Case~1.} If $\mypenalty(\matrixgeneric) = 0$, we have $\matrixgeneric = 0 $ since $\mypenalty(\matrixgeneric)=0$ if and only if $\matrixgeneric=0$ by condition (2).
	Then, 
	\begin{equation*}
	\inprod{\matrix}{\matrixgeneric} =  \dualpenalty (\matrix)\mypenalty (\matrixgeneric) =  0 .
	\end{equation*}
	In particular, $\inprod{\matrix}{\matrixgeneric} \leq \dualpenalty (\matrix)\mypenalty (\matrixgeneric)$, as desired.\\
	
	{\it Case~2.} If $\mypenalty(\matrixgeneric) \neq 0$, we rewrite
	\begin{equation*}
		\inprod{\matrix}{\matrixgeneric} = \inprod{\matrix}{\matrixgeneric}\cdot \frac{\mypenalty(\matrixgeneric) }{\mypenalty(\matrixgeneric) }= \inprod{\matrix}{\frac{\matrixgeneric}{\mypenalty(\matrixgeneric)}}\cdot \mypenalty(\matrixgeneric).
	\end{equation*}
	Next we show that $\inprod{\matrix}{\frac{\matrixgeneric}{\mypenalty(\matrixgeneric)}}\leq \dualpenalty(\matrix)$ by two observations. The first observation is that since \spanMatrix\ is a real vector space,
	\begin{equation*}
	\frac{\matrixgeneric}{\mypenalty(\matrixgeneric)} \in \spanMatrix.
	\end{equation*}
	The second observation is that $\mypenalty(\matrixgeneric) \in (0, \infty)$ in Case 2, and therefore for regularizers that are positive homogeneous of degree one as specified in condition (3),
	\begin{equation*}
	\mypenalty\Big(\frac{\matrixgeneric}{\mypenalty(\matrixgeneric)}\Big) =\frac{\mypenalty(\matrixgeneric)}{\mypenalty(\matrixgeneric)}=1.
	\end{equation*}	
	Combining the two observations, we have
	\begin{equation*}
		\inprod{\matrix}{\frac{\matrixgeneric}{\mypenalty(\matrixgeneric)}} \leq \sup\{ \inprod{\matrix}{\Lambda_1} \mid \Lambda_1 \in \spanMatrix, \mypenalty(\Lambda_1) \leq 1 \} = \dualpenalty({\matrix}).
	\end{equation*}
	Since $\mypenalty(\matrixgeneric)\in(0, \infty)$, it follows that $\inprod{\matrix}{\matrixgeneric}\leq  \dualpenalty(\matrix)\mypenalty(\matrixgeneric)$.
\end{proof}

We now proceed to the proof of Theorem~1.
\begin{proof}[of Theorem~1] The proof consists of three steps. First, we link the objective function of MRLEs with  the regularized Kullback-Leibler loss. Second, we use the convexity of $\matrix\mapsto\mycondition$ to obtain an enhanced basic inequality. Third, we use the properties of $u$ and $\dualpenalty$ shown in Lemma~A1 to bound the empirical process and conclude the proof.\\
	
	{\it Step~1: (Regularized Kullback-Leibler Loss)} We first show that the MRLE $\matrixest$ defined in (1) satisfies
	\begin{equation*}
		\matrixest \in \argmin_{\matrix \in \Matrix}\big\{\empiricalrisk(\matrix) + \tuningparameter \mypenalty (\matrix) \big\}.
	\end{equation*}
	
	Recall that MRLEs are defined as
	\begin{equation*}
		\matrixest  \in  \argmin_{ \matrix \in \Matrix}\big\{ -\log f_\matrix(\data) + \tuningparameter \mypenalty (\matrix) \big\}.
	\end{equation*}
	Since adding constant terms does not alter the estimator, we rewrite the definition as
	\begin{equation*}
	\label{eqn: regularized MLE 2}
	\matrixest \in \argmin_{\matrix \in \Matrix}\big\{\log\density{\matrixtrue}(\data)-\log f_\matrix(\data) + \tuningparameter \mypenalty (\matrix) \big\}.
	\end{equation*}
	The term $\log\density{\matrixtrue}(\data)-\log f_\matrix(\data)$ is the empirical version of the Kullback-Leibler loss defined in Equation~(4). Hence we obtain an equivalent definition of MRLE in the form of
	\begin{equation*}
		\matrixest \in \argmin_{\matrix \in \Matrix}\big\{\empiricalrisk(\matrix) + \tuningparameter \mypenalty (\matrix) \big\}.
	\end{equation*}
	This concludes Step~1.\\
	
	{\it Step~2: (Enhanced Basic Inequality)} We use Step~1 and the convexity of $\matrix\mapsto\mycondition$ to derive the enhanced basic inequality
	\begin{equation*}
	\populationrisk(\matrixest)  \leq  \tuningparameter\mypenalty(\matrixtrue)-\tuningparameter\mypenalty(\matrixest) + \inprod{\gradient}{\matrixest}+\inprod{\gradient}{-\matrixtrue}.
	\end{equation*}
	
	The proof of this inequality has two ingredients. The first ingredient is that $\matrixest$ minimizes $\empiricalrisk(\matrix)+ 	\tuningparameter \mypenalty (\matrix)$, as derived in Step~1. Hence, in particular,
	\begin{equation*}
	\empiricalrisk(\matrixest)+ \tuningparameter \mypenalty (\matrixest) \leq \empiricalrisk(\matrixtrue)+ \tuningparameter \mypenalty (\matrixtrue).
	\end{equation*}
	Rearranging the inequality yields
	\begin{equation*}
	\empiricalrisk(\matrixest)  \leq  \empiricalrisk (\matrixtrue)+ \tuningparameter\mypenalty(\matrixtrue)-\tuningparameter\mypenalty(\matrixest) .
	\end{equation*}
	
	The second ingredient is the convexity of $\mycondition$ in $\matrix$. Since $\populationrisk(\matrix)-\empiricalrisk(\matrix) =( \E_{\matrixtrue}\log\density{\matrixtrue}-\log\density{\matrixtrue})+(\mycondition)$, the convexity implies that the function $\populationrisk(\matrix)-\empiricalrisk(\matrix)$ is also convex in  $\matrix$. Hence, it holds that
	\begin{equation*}
	\populationrisk(\matrixtrue) -\empiricalrisk(\matrixtrue) \geq \populationrisk(\matrixest) -\empiricalrisk(\matrixest) + \inprod{\gradient}{\matrixtrue-\matrixest},
	\end{equation*}
	where $\gradient$ is any subgradient of $\populationrisk(\matrix)-\empiricalrisk(\matrix)$ at $\matrixest$.
	Rearranging the equality leads to 
	\begin{equation*}
		\populationrisk(\matrixest) -\empiricalrisk(\matrixest) \leq \populationrisk(\matrixtrue) -\empiricalrisk(\matrixtrue) + \inprod{\gradient}{\matrixest-\matrixtrue}.
	\end{equation*}
	
	Combining the two ingredients and doing some algebra yield 
	\begin{equation*}
		\populationrisk(\matrixest) \leq \populationrisk (\matrixtrue)+ \tuningparameter\mypenalty(\matrixtrue)-\tuningparameter\mypenalty(\matrixest) + \inprod{\gradient}{\matrixest-\matrixtrue}.
	\end{equation*}
	By the definition of $\populationrisk(\matrixtrue)$, we also find
	\begin{equation*}
	\populationrisk(\matrixtrue)  =  \E_{\matrixtrue}\text{log}\Big(\frac{f_{\matrixtrue}(x)}{f_{\matrixtrue}(x)}\Big)=0  .
	\end{equation*}
	We can thus remove $\populationrisk (\matrixtrue)$ from the inequality above and find	\begin{align*}
	\populationrisk(\matrixest)  \leq  \tuningparameter\mypenalty(\matrixtrue)-\tuningparameter\mypenalty(\matrixest) +\inprod{\gradient}{\matrixest-\matrixtrue}
 =  \tuningparameter\mypenalty(\matrixtrue)-\tuningparameter\mypenalty(\matrixest) +\inprod{\gradient}{\matrixest}+\inprod{\gradient}{-\matrixtrue}.
	\end{align*}
	This concludes Step~2.\\
	
	{\it Step~3: (Bound for the Empirical Process Term)} We show that on the event where $\tuningparameter  \geq  \dualpenalty(\gradient)  $,  it holds that
	\begin{equation*}
	\populationrisk(\matrixest)  \leq  \tuningparameter \mypenalty (\matrixtrue)+\tuningparameter \mypenalty (-\matrixtrue)  .
	\end{equation*}
	
	To this end, we first apply Lemma~A1 to the last two terms on right-hand side of the result in Step~2 and find
	\begin{equation*}
	\populationrisk(\matrixest) \leq  \tuningparameter\mypenalty(\matrixtrue)-\tuningparameter\mypenalty(\matrixest)+ \dualpenalty(\gradient)\mypenalty(\matrixest)+\dualpenalty(\gradient)\mypenalty(-\matrixtrue)  .
	\end{equation*}
	Rearranging the terms of the right-hand side yields
	\begin{equation*}
	\populationrisk(\matrixest)  \leq  \tuningparameter\mypenalty(\matrixtrue)+\dualpenalty(\gradient) \mypenalty(-\matrixtrue) + \big(-\tuningparameter +\dualpenalty(\gradient) \big)\mypenalty(\matrixest) .
	\end{equation*}
	Because $u(\cdot)$ is non-negative by definition, the inequality $\tuningparameter  \geq  \dualpenalty(\gradient)  $ implies
	\begin{equation*}
	\big(-\tuningparameter +\dualpenalty(\gradient) \big)\mypenalty(\matrixest)  \leq  0 ,
	\end{equation*}
	and 
	\begin{equation*}
	\dualpenalty(\gradient) \mypenalty(-\matrixtrue)  \leq \tuningparameter \mypenalty (-\matrixtrue). 
	\end{equation*}
	Combining the last three inequalities, we thus find on the event where $\tuningparameter  \geq  \dualpenalty(\gradient)  $ the inequality
	\begin{equation*}
		\populationrisk(\matrixest) \leq \tuningparameter \mypenalty (\matrixtrue)+\tuningparameter \mypenalty (-\matrixtrue)  .
	\end{equation*}
	This concludes the Step~3 and thus completes the proof of Theorem~1.
\end{proof}

%% file: Sections/OtherProofs.tex

\section{Proof of Example Results}
\label{appendix: other}
\begin{proof}[of Lemma~1] This lemma is a specification of Theorem~1 to tensor response regression with array normal noise. The proof consists of three steps. First, we obtain the explicit form of the empirical and population version of Kullback-Leibler loss, $\empiricalrisk(\matrix)$ and $\populationrisk(\matrix)$. Second, we derive the gradient of $\conditiona$ at $\matrixest$, denoted as $\gradient$. At last, we apply Theorem~1 with the derived explicit forms of $\populationrisk(\matrixest)$ and $\gradient$ to conclude the proof. \\
	
	{\it Step~1.} We first derive the explicit form of the empirical and population version of Kullback-Leibler loss. Plugging the array normal density into the empirical Kullback-Leibler divergence between conditional densities $\density{\matrix}$ and $\density{\matrixtrue}$ yields 
	\begin{align*}
	\empiricalrisk(\matrix) =&\, \frac{1}{2}\sumn \big( \norm{( \tensorerrorformula )\times \Sigma^{-1/2}}^2 - \norm{( \tensorerror )\times \Sigma^{-1/2}}^2\big) \\
	=&\, \frac{1}{2}\sumn \big( \norm{(\tensorerror+ \tensorpredictaformula)\times \Sigma^{-1/2}}^2 -\\
	&\norm{(\tensorerror) \times \Sigma^{-1/2}}^2 \big).
	\end{align*}
	Lemma~D1 shows that $(\tensorerror+ \tensorpredictaformula)\times \Sigma^{-1/2}= (\tensorerror) \times \Sigma^{-1/2}+( \tensorpredictaformula)\times \Sigma^{-1/2} $. We can thus reorganize the above equation as
	\begin{align*}
		\empiricalrisk(\matrix) =&\, \frac{1}{2}\sumn \big( \norm{(\tensorerror\big)\times \Sigma^{-1/2} + (\tensorpredictaformula)\times \Sigma^{-1/2}}^2 -\\
		&\norm{(\tensorerror) \times \Sigma^{-1/2}}^2 \big).	
	\end{align*}
	We expand the first array norm as shown in Lemma~D2 and get 
	\begin{align*}
		\empiricalrisk(\matrix) =&\,\frac{1}{2} \sumn \big( \norm{(\tensorerror)\times \Sigma^{-1/2}}^2 +\norm{(\tensorpredictaformula )\times \Sigma^{-1/2}}^2\\
		&+2\inprod{(\tensorerror)\times \Sigma^{-1/2}}{(\tensorpredictaformula)\times \Sigma^{-1/2}}-\\
		&  \norm{(\tensorerror)\times \Sigma^{-1/2}}^2 \big).		
	\end{align*}
	Canceling the first and last term of $\empiricalrisk(\matrix)$ yields
	\begin{align*}
	\empiricalrisk(\matrix) =&\,\frac{1}{2} \sumn \big( \norm{(\tensorpredictaformula )\times \Sigma^{-1/2}}^2\\
	&+2\inprod{(\tensorerror)\times \Sigma^{-1/2}}{(\tensorpredictaformula )\times \Sigma^{-1/2}} \big).		
\end{align*}
Taking the expectation of $\empiricalrisk(\matrix)$ with respect to $\regressionoutputi$ conditioning on $\regressioncovariatesi$ gives the conditional Kullback-Leibler divergence 
	\begin{align*}
	\populationrisk(\matrix) =&\,\frac{1}{2} \sumn\E_\matrixtrue \big( \norm{(\tensorpredictaformula )\times \Sigma^{-1/2}}^2\\
	&+2\inprod{(\tensorerror)\times \Sigma^{-1/2}}{(\tensorpredictaformula )\times \Sigma^{-1/2}} \big)\\
	=&\,\frac{1}{2} \sumn\norm{(\tensorpredictaformula )\times \Sigma^{-1/2}}^2\\
	&+\sumn\E_\matrixtrue\inprod{(\tensorerror)\times \Sigma^{-1/2}}{(\tensorpredictaformula )\times \Sigma^{-1/2}}, 	
	\end{align*}
	where $\E_{\matrixtrue}$ denotes the conditional expectation conditioning on $\regressioncovariatesi$.
	The expectation of the inner product term becomes
	\begin{align*}
		&\E_\matrixtrue \inprod{(\tensorerror)\times \Sigma^{-1/2}}{(\tensorpredictaformula )\times \Sigma^{-1/2}} \\
		=&\,\sums{i_2=1}{\dimension_2}{i_p=1}{\dimension_p}\E_\matrixtrue\big((\tensorerror)\times \Sigma^{-1/2}\big)\tensorindexgeneric{i_2}{i_\tensordimension}\cdot\big((\tensorpredictaformula)\times \Sigma^{-1/2}\big)\tensorindexgeneric{i_2}{i_\tensordimension}.
	\end{align*}
Further, since $\E_\matrixtrue\big( \regressionoutputi \big)=\matrixtrue\times_1\regressioncovariatesi$,
	\begin{align*}
		&\E_\matrixtrue\big((\tensorerror)\times \Sigma^{-1/2}\big)\tensorindexgeneric{i_2}{i_\tensordimension}\\
		=&\,\E_\matrixtrue\big( \sumsnew{j}{2}(\tensorerror)\tensorindexgeneric{j_2}{j_\tensordimension}\cdot\elementinverse{2}{i_2j_2}\dots\elementinverse{\tensordimension}{i_\tensordimension j_\tensordimension}\big)\\
		=&\, \sumsnew{j}{2}\E_\matrixtrue(\tensorerror)\tensorindexgeneric{j_2}{j_\tensordimension}\cdot\elementinverse{2}{i_2j_2}\dots\elementinverse{\tensordimension}{i_\tensordimension j_\tensordimension}\\
		=&\, 0,
	\end{align*}
where $\elementinverse{k}{ij}$ denotes the $(i,j)$th element of $A_k^{-1}$.
	Thus, $\E_\matrixtrue \inprod{(\tensorerror)\times \Sigma^{-1/2}}{(\tensorpredictaformula )\times \Sigma^{-1/2}} =0$. We then find
	\begin{equation*}
		\populationrisk(\matrix) =\frac{1}{2} \sumn\norm{(\tensorpredictaformula)\times \Sigma^{-1/2}}^2.	
	\end{equation*}
	By the definition of the mode-$1$ product, we write $\tensorpredictaformula=\matrixdeltaformula \tensorproduct{1} \regressioncovariatesi$ and conclude that the conditional Kullback-Leibler divergence of tensor regression with array normal noise has the explicit form of
	\begin{equation*}
	\populationrisk(\matrix)= \frac{1}{2} \sumn \norm{\matrixdeltaformula \tensorproduct{1} \regressioncovariatesi \times \Sigma^{-1/2}}^2,
	\end{equation*}
	which is the prediction error.\\ 
	
	{\it Step~2.} We derive the explicit form of $\gradient$. With $\populationrisk(\matrix)$ and $\empiricalrisk(\matrix)$ derived in Step~1, we obtain
	\begin{align*}
		\conditionaest =&\, \frac{1}{2} \sumn \norm{\matrixdelta \tensorproduct{1} \regressioncovariatesi \times \Sigma^{-1/2}}^2 - \frac{1}{2} \sumn \big( \norm{(\tensorpredicta )\times \Sigma^{-1/2}}^2\\
		&+2\inprod{(\tensorerror)\times \Sigma^{-1/2}}{(\tensorpredicta )\times \Sigma^{-1/2}} \big).
	\end{align*}
	Canceling the first and second term in the above equality yields
	\begin{equation*}
		\conditionaest = -\sumn\inprod{(\tensorerror)\times \Sigma^{-1/2}}{(\tensorpredicta)\times \Sigma^{-1/2}} .
	\end{equation*}
	Let $\Wi:=\inprod{(\tensorerror)\times \Sigma^{-1/2}}{(\tensorpredicta )\times \Sigma^{-1/2}} $. We write $\conditionaest = -\sumn\Wi$ and find 
	\begin{equation*}
	\gradient = -\sumn\Wigradient,
	\end{equation*}
	where $\Wigradient$ denotes the gradient of $\mathcal{W}^i$ at $\matrixest$. 
	
	Next we derive $\Wigradient$. Expanding $\Wi$ by the definition of tensor inner product, we get
	\begin{equation*}
		\Wi =\sums{i_2=1}{\dimension_2}{i_\tensordimension=1}{\dimension_\tensordimension}\big( (\tensorerror)\times \Sigma^{-1/2} \big)\tensorindexgeneric{i_2}{i_\tensordimension}\big( \matrixdelta \tensorproduct{1}\regressioncovariatesi \times \Sigmainversequareroot  \big)\tensorindexgeneric{i_2}{i_\tensordimension}.
	\end{equation*}
	Expanding the tensor operations in the second term yields
	\begin{equation*}
	\Wi =\sumsnew{i}{2}\big( (\tensorerror)\times \Sigma^{-1/2} \big)\tensorindexgeneric{i_2}{i_\tensordimension}\cdot\big( \sumsnew{j}{1}\matrixdelta_{j_1, \dots, j_p} z^i_{1 j_1}\elementinverse{2}{i_2j_2}\dots\elementinverse{p}{i_pj_p} \big).
	\end{equation*}
	The partial derivative of $\mathcal{W}^i$ with regarding to $\matrixest_{m_1, \dots, m_p}$ is
	\begin{align*}
	\frac{\partial \mathcal{W}^i}{\partial \matrixest_{m_1, \dots, m_p}} =&\, \sumsnew{i}{2}\big( (\tensorerror)\times \Sigma^{-1/2} \big)\tensorindexgeneric{i_2}{i_\tensordimension}( -z^i_{1 m_1}\elementinverse{2}{i_2m_2}\dots\elementinverse{p}{i_pm_p}).
	\end{align*}
	Here $z^i_{1 m_1}$ is the $(m_1,1)$th entry of $\regressioncovariatesit$, $\elementinverse{k}{i_km_k}$ is the $(m_k, i_k)$th entry of $(A_k^{-1})^\top$,   $k\in\{2, \dots, p\}$. Let $( \Sigma^{-1/2})^\top := \{( A_2^{-1})^\top, \dots, ( A_p^{-1})^\top\}$, we obtain
	\begin{equation*}
	\frac{\partial \mathcal{W}^i}{\partial \matrixest_{m_1, \dots, m_p}} =-\big( (\tensorerror) \times \Sigmainversequareroot  \times (\Sigmainversequareroot)^\top\tensorproduct{1} \regressioncovariatesit \big)_{m_1, \dots, m_p}.
	\end{equation*}
	Hence, $\Wigradient$ has the form
	\begin{equation*}
		\Wigradient= -(\tensorerror) \times \Sigmainversequareroot \times (\Sigmainversequareroot)^\top\tensorproduct{1} \regressioncovariatesit .
	\end{equation*}
	Plugging the explicit form of $\Wigradient$ into the equation of $\gradient$ yields
	\begin{equation*}
		\gradient = \sumn(\tensorerror) \times \Sigmainversequareroot \times (\Sigmainversequareroot)^\top \tensorproduct{1} \regressioncovariatesit.
	\end{equation*}
Under the assumed tensor regression model, $\tensorerror= \regressionnoisei$, so that
\begin{equation*}
\gradient= \sumn \big( \regressionnoisei \times \Sigmainversequareroot  \times (\Sigmainversequareroot)^\top\tensorproduct{1} \regressioncovariatesit\big).
\end{equation*}
Expanding the Tucker product as a sequence of mode products gives
\begin{equation*}
\gradient= \sumn \big( \regressionnoisei \tensorproduct{1} A_2^{-1} \tensorproduct{2} \dots \tensorproduct{p-1} A_{p}^{-1}  \tensorproduct{1} (A_2^{-1})^\top \tensorproduct{2} \dots \tensorproduct{p-1} (A_{p}^{-1})^\top\tensorproduct{1} \regressioncovariatesit\big).
\end{equation*}
The properties of the tensor mode product include $(\mathcal{T}\times_i W)\times_j V = (\mathcal{T}\times_j V)\times_i W =\mathcal{T}\times_i W\times_j V$ for $i\neq j$ and $(\mathcal{T}\times_i W)\times_i V = \mathcal{T}\times_i (VW)$~\citep{De2000multilinear}. Reorganizing the above equation yields
\begin{equation*}
\gradient= \sumn \big( \regressionnoisei \tensorproduct{1} \big((A_2^{-1})^\top A_2^{-1}\big) \tensorproduct{2} \dots \tensorproduct{p-1} \big((A_p^{-1})^\top A_p^{-1}\big) \tensorproduct{1} \regressioncovariatesit\big).
\end{equation*}
Since $(A_k^{-1})^TA_k^{-1} = ( A_k A_k^T )^{-1} = \Sigma_k^{-1}$, 
\begin{equation*}
\gradient= \sumn \big( \regressionnoisei \times \Sigma^{-1} \tensorproduct{1} \regressioncovariatesit \big).
\end{equation*}
	
	{\it Step~3.} We apply Theorem~1 with the explicit form of $\populationrisk(\matrixest)$ and $\gradient$, and conclude that for all $\tuningparameter \geq \dualpenalty\big(\sumn \big( \regressionnoisei  \times \Sigma^{-1}\tensorproduct{1} \regressioncovariatesit \big)\big)$, it holds that
	\begin{equation*}
	\frac{1}{2}  \sumn\norm{\matrixdelta \tensorproduct{1} \regressioncovariatesi \times \Sigma^{-1/2}}^2 \leq \tuningparameter\mypenalty(\matrixtrue)+ \tuningparameter\mypenalty(-\matrixtrue).
	\end{equation*}
	This completes the proof.
\end{proof}

\begin{proof}[of Lemma~2]
This lemma is a specification of Theorem~1 to generalized linear tensor regression. The proof consists of three steps. First, we obtain the explicit form of the empirical and population version of the Kullback-Leibler loss, $\empiricalrisk(\matrix)$ and $\populationrisk(\matrix)$. Second, we derive the gradient of $\conditiona$ at $\matrixest$, denoted as $\gradient$. At last, we apply Theorem~1 with the derived explicit forms of $\populationrisk(\matrixest)$ and $\gradient$ to conclude the proof. \\

{\it Step~1.} We first derive the explicit form of the empirical and population version of the Kullback-Leibler loss. Plugging the conditional density of $\lmoutputi\mid\lmcovariatesi$ into the empirical Kullback-Leibler divergence yields 
\begin{equation*}
	\empiricalrisk(\matrix)= \sumn\big( \frac{\lmoutputi\inprod{\matrixtrue}{\lmcovariatesi}-b(\inprod{\matrixtrue}{\lmcovariatesi})}{\alpha} +c(\lmoutputi, \alpha)- \frac{\lmoutputi\inprod{\matrix}{\lmcovariatesi}-b(\inprod{\matrix}{\lmcovariatesi})}{\alpha} - c(\lmoutputi, \alpha)\big).
\end{equation*}
Canceling the term $c(\lmoutputi, \alpha)-c(\lmoutputi, \alpha)$ and reorganizing the terms on the right-hand side yields
\begin{equation*}
	\empiricalrisk(\matrix)= \frac{1}{\alpha}\sumn\big( \lmoutputi\inprod{\matrixtrue-\matrix}{\lmcovariatesi} - b(\inprod{\matrixtrue}{\lmcovariatesi}) + b(\inprod{\matrix}{\lmcovariatesi}) \big).
\end{equation*}
Taking the expectation of $\empiricalrisk(\matrix)$ with respect to $\lmoutputi$ conditioning on $\regressioncovariatesi$, we obtain the conditional Kullback-Leibler divergence in the form of
\begin{equation*}
	\populationrisk(\matrix) = \frac{1}{\alpha}\sumn\E_{\matrixtrue}\big( \lmoutputi\inprod{\matrixtrue-\matrix}{\lmcovariatesi} - b(\inprod{\matrixtrue}{\lmcovariatesi}) + b(\inprod{\matrix}{\lmcovariatesi}) \big),	
\end{equation*}
where $\E_{\matrixtrue}$ denotes the conditional expectation conditioning on $\regressioncovariatesi$. As $\E_{\matrixtrue}(\lmoutputi) = g^{-1}( \inprod{\matrixtrue}{\lmcovariatesi})$, we find that the conditional Kullback-Leibler divergence of the generalized linear tensor regression has the explicit form of
\begin{equation*}
\populationrisk(\matrix)= \frac{1}{\alpha}\sumn\big(g^{-1}( \inprod{\matrixtrue}{\lmcovariatesi})\cdot\inprod{\matrixtrue-\matrix}{\lmcovariatesi} - b(\inprod{\matrixtrue}{\lmcovariatesi}) + b(\inprod{\matrix}{\lmcovariatesi}) \big).
\end{equation*}

{\it Step~2.} We derive the explicit form of $\gradient$. With $\populationrisk(\matrix)$ and $\empiricalrisk(\matrix)$ derived in Step~1, we obtain
\begin{align*}
	\conditionaest =&\, \frac{1}{\alpha}\sumn\big(g^{-1}( \inprod{\matrixtrue}{\lmcovariatesi})\cdot\inprod{\matrixtrue-\matrixest}{\lmcovariatesi} - b(\inprod{\matrixtrue}{\lmcovariatesi}) + b(\inprod{\matrixest}{\lmcovariatesi}) -\\
	& \lmoutputi\inprod{\matrixtrue-\matrixest}{\lmcovariatesi} + b(\inprod{\matrixtrue}{\lmcovariatesi}) - b(\inprod{\matrixest}{\lmcovariatesi}) \big).
\end{align*}
Canceling the terms involving $b(\cdot)$ in the above equality yields
\begin{align*}
\conditionaest =&\, \frac{1}{\alpha}\sumn\big(g^{-1}( \inprod{\matrixtrue}{\lmcovariatesi})\cdot\inprod{\matrixtrue-\matrixest}{\lmcovariatesi} -\lmoutputi\inprod{\matrixtrue-\matrixest}{\lmcovariatesi} \big)\\
=&\, \frac{1}{\alpha}\sumn \big(g^{-1}(\inprod{\matrixtrue}{\lmcovariatesi})-\lmoutputi\big)\inprod{\matrixtrue-\matrixest}{\lmcovariatesi}. 
\end{align*}
Expanding the tensor inner product by definition, we find that the gradient $\gradient$ has the explicit form
\begin{align*}
	\gradient =&\, \frac{1}{\alpha}\sumn\big(g^{-1}(\inprod{\matrixtrue}{\lmcovariatesi})-\lmoutputi\big)(-\lmcovariatesi)\\
	=&\, \frac{1}{\alpha}\sumn\big(\lmoutputi-g^{-1}(\inprod{\matrixtrue}{\lmcovariatesi})\big)\lmcovariatesi. 
\end{align*}

{\it Step~3.} We apply Theorem~1 with the explicit form of $\populationrisk(\matrixest)$ and $\gradient$ and conclude that for all $\tuningparameter \geq \dualpenalty\big( \frac{1}{\alpha}\sumn\big(\lmoutputi-g^{-1}(\inprod{\matrixtrue}{\lmcovariatesi})\big)\lmcovariatesi\big)$, it holds that
\begin{equation*}
\populationrisk(\matrixest) \leq \tuningparameter\mypenalty(\matrixtrue)+ \tuningparameter\mypenalty(-\matrixtrue)
\end{equation*}
with Kullback-Leibler loss $\populationrisk(\matrixest)= \frac{1}{\alpha}\sumn\big(g^{-1}( \inprod{\matrixtrue}{\lmcovariatesi})\cdot\inprod{\matrixtrue-\matrixest}{\lmcovariatesi} - b(\inprod{\matrixtrue}{\lmcovariatesi}) + b(\inprod{\matrixest}{\lmcovariatesi}) \big)$. To convey the idea that the choice of the regularization parameter is in the form of noise, we write the lower bound of the regularization parameter as $\dualpenalty\big( \frac{1}{\alpha}\sumn\big(\lmoutputi-\E_{\matrixtrue}(\lmoutputi)\big)\lmcovariatesi \big)$.
\end{proof}

\begin{proof}[of Lemma~3]
	The proof consists of three steps. First, we obtain the explicit form of $\empiricalrisk(\matrix)$ and $\populationrisk(\matrix)$ in exponential trace models. Second, we derive the gradient of $\conditiona$ at $\matrixest$. At last, we apply Theorem~1 with the derived explicit forms. \\
	
	{\it Step~1.} In the exponential trace model, it holds that
	\begin{equation*}
	\sumn\text{log}\big(f_{\matrix}(\datasi)\big) = -\sumn\inprod{\matrix}{T(\datasi)}-n\normtotal(\matrix) .
	\end{equation*}
	The definition of inner product implies that $\inprod{\matrix}{T(\datasi)}= \inprod{T(\datasi)}{\matrix}$. Therefore, we write
	\begin{equation*}
	\sumn\text{log}\big(f_{\matrix}(\datasi)\big) = -\sumn\inprod{T(\datasi)}{\matrix}-n\normtotal(\matrix) .
	\end{equation*}
	Plugging the log-likelihood into the empirical Kullback-Leibler loss yields
	\begin{equation*}
	\empiricalrisk(\matrix) = -\sumn\inprod{T(\datasi)}{\matrixtrue}-n\normtotal(\matrixtrue) + \sumn\inprod{T(\datasi)}{\matrix}+n\normtotal(\matrix) .
	\end{equation*}
	Since inner product $\inprod{\cdot }{\cdot}$ is linear, we write
	\begin{equation*}
	\empiricalrisk(\matrix) = \inprod{\sumn T(\datasi)}{\matrix-\matrixtrue}-n\normtotal(\matrixtrue) +n\normtotal(\matrix) .
	\end{equation*}
	The population loss $\populationrisk(\matrix)$ is the expectation of the empirical loss $\empiricalrisk(\matrix)$. We thus obtain
	\begin{equation*}
	\populationrisk(\matrix)  = \E_{\matrixtrue}\big( \inprod{\sumn T(\datasi)}{\matrix-\matrixtrue}-n\normtotal(\matrixtrue) +n\normtotal(\matrix) \big).
	\end{equation*}
	Again, $\inprod{\cdot }{\cdot}$ is linear, and we find
	\begin{equation*}
	\populationrisk(\matrix)  =  \inprod{\sumn\E_{\matrixtrue}T(\datasi)}{\matrix-\matrixtrue}-n\normtotal(\matrixtrue) +n\normtotal(\matrix). 
	\end{equation*}
	
	{\it Step~2.} We derive the explicit form of $\gradient$ in the exponential trace model. With the forms of $\populationrisk(\matrix)$ and $\empiricalrisk(\matrix)$ derived in Step~1, we obtain
	\begin{equation*}
		\conditionaest=\big( \inprod{\sumn\E_{\matrixtrue}T(\datasi)}{\matrixest-\matrixtrue}-n\normtotal(\matrixtrue) +n\normtotal(\matrixest)\big)- \big(\inprod{\sumn T(\datasi)}{\matrixest-\matrixtrue}-n\normtotal(\matrixtrue) +n\normtotal(\matrixest) \big).
	\end{equation*}
	Applying the linearity of inner product and canceling $-n\normtotal(\matrixtrue) +n\normtotal(\matrixest)+n\normtotal(\matrixtrue) -n\normtotal(\matrixest)$, we obtain 
	\begin{equation*}
		\conditiona = \inprod{\gmgradient}{\matrixest-\matrixtrue}. 
	\end{equation*}
	The definition of inner product implies that 
	\begin{equation*}
		\gradient = \gmgradient.
	\end{equation*}
	
	{\it Step~3.}
	Plugging the explicit form of $\gradient$ into the lower bound for the regularization parameter in Theorem~1, we conclude for all $\tuningparameter  \geq  \dualpenalty\big(\gmgradient\big) $ the inequality
	\begin{equation*}
		\populationrisk(\matrixest) \leq \tuningparameter \mypenalty (\matrixtrue)+\tuningparameter \mypenalty (-\matrixtrue)  
	\end{equation*}
	with Kullback-Leibler loss $	\populationrisk(\matrixest)  =  \inprod{\sumn\E_{\matrixtrue}T(\datasi)}{\matrixest-\matrixtrue}-n\normtotal(\matrixtrue) +n\normtotal(\matrixest)$.
\end{proof}

%% file: Sections/EP.tex

\newcommand{\yp}[1]{\textcolor{red}{\bf #1}}
\section{Empirical Process Terms}\label{appendix:EP}
\begin{lemma}[control of the empirical term for tensor response regression]\label{EP1}
  Let  $z_j^i$ be the $j$th entry of row vector $\regressioncovariatesi$. Suppose $\sum_{i=1}^{n}(z_j^i)^2/n=1$, and $(\Sigma_k^{-1})_{i_ki_k} = h_k^2$, $h_k>0$, for $k \in \{2, \dots, p\}$, $i_k \in \{1, \dots, b_k\}$. Consider the case with zero-mean array normal noise in~Section~3.1 and $\mypenalty(\matrix) := \sums{i_1=1}{b_1}{i_p=1}{b_p}\abs{\matrix_{\tensorindexgeneric{i_1}{i_p}}}$. For all $t>0$ and $\tuningparameter_0 = \big(\prod_{k=2}^{p}h_k\big)\sqrt{2n\big(t^2+\log(\prod_{j=1}^{p}b_j)\big)}$, it holds that
  \begin{equation*}
  P\big(\dualpenalty\big(\sum_{i=1}^{n}\big( \regressionnoisei \times \Sigma^{-1}\tensorproduct{1} \regressioncovariatesit  \big)\big)\leq \tuningparameter_0\big) \geq 1-2\exp(-t^2).
  \end{equation*}
\end{lemma}

\begin{proof}[of Lemma~\normalfont{C1}]
The proof consists of two steps. First, we plug in $\dualpenalty(\cdot)$ and re-express the target probability in terms of random variables $\big(\sum_{i=1}^{n}\big( \regressionnoisei \times \Sigma^{-1}\tensorproduct{1} \regressioncovariatesit  \big)_{\tensorindexgeneric{i_1}{i_\tensordimension}}$. Second, we apply the Chernoff inequality to bound the tail probabilities. \\

{\it Step~1.} For $\matrix \in \R^{\dimension_1\times \dots \times \dimension_\tensordimension}$, the dual of the regularizer  is of the form
  	\begin{equation*}
  	\dualpenalty(\matrix) := \sup\big\{  \inprod{\matrix}{\matrix'}  \mid \matrix' \in \R^{\dimension_1\times \dots \times \dimension_\tensordimension}, \mypenalty(\matrix')  \leq  1 \big\} = \max_{i_1, \dots, i_\tensordimension}\abs{\matrix_{\tensorindexgeneric{i_1}{i_\tensordimension}}}. 
  	\end{equation*}
Therefore,
  	\begin{align*}
  	&P\big(\dualpenalty\big(\sum_{i=1}^{n}\big( \regressionnoisei \times \Sigma^{-1}\tensorproduct{1} \regressioncovariatesit  \big)\big)\leq \tuningparameter_0\big)  \\
  	&= P\big(\max_{i_1, \dots, i_\tensordimension}\abs{\big(\sum_{i=1}^{n}\big( \regressionnoisei \times \Sigma^{-1}\tensorproduct{1} \regressioncovariatesit  \big)_{\tensorindexgeneric{i_1}{i_\tensordimension}}}\leq \tuningparameter_0\big)\\
  	&= 1- P\big(\max_{i_1, \dots, i_\tensordimension}\abs{\big(\sum_{i=1}^{n}\big( \regressionnoisei \times \Sigma^{-1}\tensorproduct{1} \regressioncovariatesit  \big)_{\tensorindexgeneric{i_1}{i_\tensordimension}}}> \tuningparameter_0\big)\\
  	&\geq 1-\sums{i_1=1}{b_1}{i_p=1}{b_p}P\big(\abs{\big(\sum_{i=1}^{n}\big( \regressionnoisei \times \Sigma^{-1}\tensorproduct{1} \regressioncovariatesit  \big)_{\tensorindexgeneric{i_1}{i_\tensordimension}}}> \tuningparameter_0\big).
  	\end{align*}

{\it Step~2.} We now work on the distribution of $\big(\sum_{i=1}^{n}\big( \regressionnoisei \times \Sigma^{-1}\tensorproduct{1} \regressioncovariatesit  \big)_{\tensorindexgeneric{i_1}{i_\tensordimension}}$ to bound its tail probability. We first observe that
  	\begin{align*}
  	\big(\sum_{i=1}^{n}\big( \regressionnoisei \times \Sigma^{-1}\tensorproduct{1} \regressioncovariatesit  \big)_{\tensorindexgeneric{i_1}{i_\tensordimension}}
  	&= \big(\sum_{i=1}^{n}\big( \regressionnoisei \times \Sigma^{-1/2}\times(\Sigmainversequareroot)^\top\tensorproduct{1} \regressioncovariatesit  \big)_{\tensorindexgeneric{i_1}{i_\tensordimension}}.
  	\end{align*}
  	The definition of the array normal distribution in~\cite{Hoff2011separable} shows that
  	\begin{equation*}
  	\regressionnoisei = N^i \times A,
  	\end{equation*}
  	where $N^i$  is an array of independent standard normal entries in $\R^{\dimension_2\times \dots \times \dimension_\tensordimension}$. Then,
  	\begin{equation*}
  		\regressionnoisei \times \Sigma^{-1/2} = (N^i \times A) \times \Sigma^{-1/2}.
  	\end{equation*}
  	Expanding the Tucker product in the above equation yields
  	\begin{equation*}
  			\regressionnoisei \times \Sigma^{-1/2}
  		= (N^i \tensorproduct{1}A_2\tensorproduct{2}\dots\tensorproduct{p-1}A_p)\tensorproduct{1} A_2^{-1}\tensorproduct{2} \cdots \tensorproduct{p-1}A_p^{-1}.
  	\end{equation*}
  	Further, we apply the properties of mode products on the right-hand side and find
  	\begin{equation*}
  	\regressionnoisei \times \Sigma^{-1/2}
  	= N^i \tensorproduct{1}(A_2^{-1}A_2)\tensorproduct{2}\dots\tensorproduct{p-1}(A_p^{-1}A_p)
  	 =N^i \tensorproduct{1}I^{(2)}\tensorproduct{2}\dots\tensorproduct{p-1}I^{(p)},
  	\end{equation*}
  	where $I^{(k)}$ is the identity matrix of dimension $b_k\times b_k$ for $k \in \{2, \dots, p\}$. Hence,
  	\begin{equation*}
  		\regressionnoisei \times \Sigma^{-1/2} = N^i.
  	\end{equation*}
  	The above equality can be shown by writing out the element-wise form of $\regressionnoisei \times \Sigma^{-1/2}$ as
  	\begin{align*}
  	(\regressionnoisei \times \Sigma^{-1/2})_{i_2, \dots, i_p} 
  	& = \big(N^i \tensorproduct{1}I^{(2)}\tensorproduct{2}\dots\tensorproduct{p-1}I^{(p)}\big)_{i_2, \dots, i_p}\\
  	&= \sums{j_2}{b_2}{j_p}{b_p}N^i_{j_2, \dots, j_p} I^{(2)}_{i_2j_2}\dots I^{(p)}_{i_pj_p}\\
  	&= N^i_{i_2, \dots, i_p}.
  	\end{align*}
  	Plugging the equality between $\regressionnoisei \times \Sigma^{-1/2}$ and $N^i$ into $\big(\sum_{i=1}^{n}\big( \regressionnoisei \times \Sigma^{-1/2}\times(\Sigmainversequareroot)^\top\tensorproduct{1} \regressioncovariatesit  \big)_{\tensorindexgeneric{i_1}{i_\tensordimension}}$ shows
  	\begin{align*}
  	\big(\sum_{i=1}^{n}\big( \regressionnoisei \times \Sigma^{-1}\tensorproduct{1} \regressioncovariatesit  \big)_{\tensorindexgeneric{i_1}{i_\tensordimension}}
  	& =  \big(\sum_{i=1}^{n}\big( N^i \tensorproduct{1}(A_2^{-1})^\top\tensorproduct{2}\dots\tensorproduct{p-1}(A_p^{-1})^\top \tensorproduct{1} \regressioncovariatesit\big)_{\tensorindexgeneric{i_1}{i_\tensordimension}}\\
  	&= \si\sum_{j_2=1}^{b_2}\cdots\sum_{j_p=1}^{b_p}N^i_{j_2,\dots, j_p} (A_2^{-1})^\top_{i_2j_2}\dots(A_p^{-1})^\top_{i_pj_p}z^i_{i_1}.
  	\end{align*}
  	Since $N^i$ is an array of independent standard normal entries, we find
  	\begin{equation*}
  	\big(\sum_{i=1}^{n}\big( \regressionnoisei \times \Sigma^{-1}\tensorproduct{1} \regressioncovariatesit  \big)_{\tensorindexgeneric{i_1}{i_\tensordimension}}\sim \mathcal{N}(0, \si\sum_{j_2=1}^{b_2}\dots\sum_{j_p=1}^{b_p} (z^i_{i_1}(A_2^{-1})^\top_{i_2j_2}\dots(A_p^{-1})^\top_{i_pj_p})^2),
  	\end{equation*}
  where the variance can be re-written as
  	\begin{equation*}
  	\big(\si (z^i_{i_1})^2\big)\big(\sum_{j_2=1}^{b_2}(A_2^{-1})^2_{j_2i_2}\big)\dots\big(\sum_{j_p=1}^{b_p}(A_p^{-1})^2_{j_pi_p}\big).
  	\end{equation*}
  	Observe that
  	$\sum_{j_k=1}^{b_k}(A_k^{-1})^2_{j_ki_k} = \big((A_k^\top)^{-1}(A_k)^{-1}\big)_{i_ki_k}= (\Sigma_k^{-1})_{i_ki_k} = h_k^2$, 
  	and $\sum_{i=1}^{n}(z_j^i)^2/n=1$, we show 
  	\begin{equation*}
  	\big(\sum_{i=1}^{n}\big( \regressionnoisei \times \Sigma^{-1}\tensorproduct{1} \regressioncovariatesit  \big)_{\tensorindexgeneric{i_1}{i_\tensordimension}}\sim \mathcal{N}(0, n\prod_{k=2}^{p}h_k^2).
  	\end{equation*}
  	We now apply the Chernoff bound for the normal distribution and control the stochastic term to find
  	\begin{align*}
  	& \sums{i_1=1}{b_1}{i_p=1}{b_p}P\big(\abs{\big(\sum_{i=1}^{n}\big( \regressionnoisei \times \Sigma^{-1}\tensorproduct{1} \regressioncovariatesit  \big)_{\tensorindexgeneric{i_1}{i_\tensordimension}}}> \tuningparameter_0\big)\\ 
  	&= 2\prod_{m=1}^{p}b_m\cdot P\big(\big(\sum_{i=1}^{n}\big( \regressionnoisei \times \Sigma^{-1}\tensorproduct{1} \regressioncovariatesit  \big)_{\tensorindexgeneric{i_1}{i_\tensordimension}}> \tuningparameter_0\big) \\
  	& \leq 2\prod_{m=1}^{p}b_m \cdot \exp(-\frac{\tuningparameter_0^2}{2n\prod_{k=2}^{p}h_k^2})).
  	\end{align*}  	
Plugging in $\tuningparameter_0$ yields
\begin{align*}
	&  \sums{i_1=1}{b_1}{i_p=1}{b_p}P\big(\abs{\big(\sum_{i=1}^{n}\big( \regressionnoisei \times \Sigma^{-1}\tensorproduct{1} \regressioncovariatesit  \big)_{\tensorindexgeneric{i_1}{i_\tensordimension}}}> \tuningparameter_0\big)\\ 
	& \leq 2\prod_{m=1}^{p}b_m \cdot \exp(-\frac{\prod_{k=2}^{p}h_k^2\cdot 2n (t^2+\log(\prod_{j=1}^{p}b_j)}{2n\prod_{k=2}^{p}h_k^2}) \\
	& = 2\prod_{m=1}^{p}b_m\cdot \exp(-t^2-\log(\prod_{j=1}^{p}b_j))\\
	& = 2\exp(-t^2).
\end{align*}  	
  	
We conclude that $P(\dualpenalty\big(\sum_{i=1}^{n}\big( \regressionnoisei \times \Sigma^{-1}\tensorproduct{1} \regressioncovariatesit  \big)\big)\leq \tuningparameter_0) \geq 1- 2\exp(-t^2)$ as desired.
\end{proof}

\begin{lemma}[control of the empirical term for logistic regression]\label{EP2}
	Let  $z_j^i$ be the $j$th entry of vector $\regressioncovariatesi$. Suppose $\sum_{i=1}^{n}(z_j^i)^2/n=1$ and $\mypenalty(\matrix) := \sum_{j=1}^{b_1}\abs{\matrix_j}$. For all $t>0$ and $\tuningparameter_0 = \sqrt{(1+2\max_i\{p^i(1-p^i)\})/3\cdot n(t^2+\log b_1)}$, where $p^i := \E_{\matrix^*}y^i = e^{\inprod{\matrixtrue}{\lmcovariatesi}}/\big(1+e^{\inprod{\matrixtrue}{\lmcovariatesi}}\big)$, it holds that
	\begin{equation*}
	P\big(\dualpenalty\big(\sum_{i=1}^{n}\big(\lmoutputi-p^i\big) \lmcovariatesi \big)\leq \tuningparameter_0\big) \geq 1-2\exp(-t^2).
	\end{equation*}
\end{lemma}

\begin{proof}[of Lemma~\normalfont{C2}] The proof consists of two steps. First, we plug in $\dualpenalty(\cdot)$ and re-express the target probability in terms of random variables $\sum_{i=1}^{n}\big(\lmoutputi-p^i\big) z_j^i$. Second, we apply the improved Hoeffding's inequality~\citep[Theorem 2.47]{bercu2015concentration} to bound the tail probabilities. \\
	
{\it Step~1.} For $\matrix \in \R^{b_1}$, the dual of the regularizer is of the form
	\begin{equation*}
	\dualpenalty(\matrix) := \sup\big\{  \inprod{\matrix}{\matrix'}  \mid \matrix' \in \R^{b_1}, \mypenalty(\matrix')  \leq  1 \big\} = \max_{j \in \{1, \dots, b_1\}}\abs{\matrix_{j}}. 
	\end{equation*}
	Therefore,
	\begin{align*}
	P\big(\dualpenalty\big(\sum_{i=1}^{n}\big(\lmoutputi-p^i\big) \lmcovariatesi \big)\leq \tuningparameter_0\big) 
	&=P(\max_{j\in \{1, \dots, b_1\}}\abs{\sum_{i=1}^{n}\big(\lmoutputi-p^i\big) z_j^i}\leq \tuningparameter_0)\\
	& = 1- P(\max_{j\in \{1, \dots, b_1\}}\abs{\sum_{i=1}^{n}\big(\lmoutputi-p^i\big) z_j^i}> \tuningparameter_0) \\
	& \geq 1- \sum_{j=1}^{b_1}P(\abs{\sum_{i=1}^{n}\big(\lmoutputi-p^i\big) z_j^i}> \tuningparameter_0).
	\end{align*}
	
{\it Step~2.} We now bound the tail probability of $\sum_{i=1}^{n}\big(\lmoutputi-p^i\big) z_j^i$. Let $U_i := \lmoutputi z_j^i$ and $S_n := \sum_{i=1}^{n}U_i$, we observe that
\begin{equation*}
	\sum_{i=1}^{n}\big(\lmoutputi-p^i\big) z_j^i = S_n - \E_{\matrixtrue}S_n,
\end{equation*}
where $\E_{\matrixtrue}S_n$ is the conditional expectation given $z_j^i$. The random variable $U_i $ is bounded, specifically, $\min\{0, z_j^i\} \leq U_i\leq \max\{0, z_j^i\}$. Applying the improved Hoeffding's inequality~\citep[Theorem 2.47]{bercu2015concentration} yields
	\begin{equation*}
	P(\sum_{i=1}^{n}\big(\lmoutputi-p^i\big) z_j^i> \tuningparameter_0)  = P(S_n-\E_{\matrix^*}S_n > \tuningparameter_0)
	\leq \exp\big(-\frac{3\tuningparameter_0^2}{D_n + 2V_n}\big),
	\end{equation*}
	where $D_n = \sum_{i=1}^{n}\big( \max\{0, z_j^i\} - \min\{0, z_j^i\} \big)^2 = \sum_{i=1}^{n}(z_j^i)^2 = n$, and $V_n = \E_{\matrixtrue}(S_n-\E_{\matrix^*}S_n)^2 = \sum_{i=1}^{n}(z_j^i)^2p^i(1-p^i)$. Similarly, we find
	\begin{equation*}
	P(\sum_{i=1}^{n}\big(\lmoutputi-p^i\big) z_j^i < - \tuningparameter_0)  = P(-S_n-\E_{\matrix^*}(-S_n) > \tuningparameter_0)
	\leq \exp\big(-\frac{3\tuningparameter_0^2}{D_n + 2V_n}\big).
	\end{equation*}
We are left with bounding the tail probabilities. To this end, we observe that
	\begin{align*}
	\sum_{j=1}^{b_1}P(\abs{\sum_{i=1}^{n}\big(\lmoutputi-p^i\big) z_j^i}> \tuningparameter_0) & \leq  2b_1\exp\big(-\frac{3\tuningparameter_0^2}{D_n + 2V_n}\big)\\
	& = 2b_1\exp\big(-\frac{3\tuningparameter_0^2}{n + 2\sum_{i=1}^{n}(z_j^i)^2p^i(1-p^i)}\big)\\
	& \leq 2b_1\exp\big(-\frac{3\tuningparameter_0^2}{n + 2n\max_i\{p^i(1-p^i)\}}\big).
	\end{align*}
	Plugging in $\tuningparameter_0$ yields
	\begin{align*}
	& \sum_{j=1}^{b_1}P(\abs{\sum_{i=1}^{n}\big(\lmoutputi-p^i\big) z_j^i}> \tuningparameter_0) \\
	&\leq  2b_1\exp\big(-\frac{(1+2\max_i\{p^i(1-p^i)\})\cdot n(t^2+\log b_1)}{n + 2n\max_i\{p^i(1-p^i)\}}\big)\\
	& = 2b_1\exp\big(-t^2 -\log b_1\big)\\
	& = 2\exp\big(-t^2\big).
	\end{align*}
We conclude that $P\big(\dualpenalty\big(\sum_{i=1}^{n}\big(\lmoutputi-p^i\big) \lmcovariatesi \big)\leq \tuningparameter_0\big) \geq 1-2\exp(-t^2)$ as desired.
\end{proof}

\begin{lemma}[control of the empirical term for gaussian graphical model]\label{EP3}
	Consider $\mypenalty(\matrix) := \sum_{i_1=1}^{p}\sum_{i_2=1}^{p}\abs{\matrix_{i_1i_2}}$ and $\data^i \sim \mathcal{N}(0, (\matrixtrue)^{-1})$. For $0 <t <\sqrt{n/4-\log\big(p(p-1)\big)}$, and $\tuningparameter_0 = 80\max_k\big((\matrixtrue)^{-1}\big)_{kk}\sqrt{n\big(t^2 + \log\big(p(p-1)\big)\big)}$, it holds that 
	\begin{equation*}
	P\big(\dualpenalty\big(\sum_{i=1}^{n}\big((\matrixtrue)^{-1} - \data^i(\data^i)^\top\big)\big)\leq \tuningparameter_0\big) \geq 1-4\exp(-t^2).
	\end{equation*}
\end{lemma}

\begin{proof}[of Lemma~\normalfont{C3}] The proof consists of two steps. First, we plug in $\dualpenalty(\cdot)$ and re-express the target probability in terms of random variables $\big(\widehat{\matrix}^{-1} - (\matrixtrue)^{-1}\big)_{i_1i_2}$, where $\widehat{\matrix}^{-1} := \sum_{i=1}^{n}\data^i(\data^i)^\top/n$ is the sample covariance matrix. Second, we apply the concentration result for sample covariance matrix~\citep[Lemma 1]{Ravikumar2011high} to bound the tail probabilities. \\
	
{\it Step~1.} For $\matrix \in \R^{p\times p}$, the dual of the regularizer is of the form
	\begin{equation*}
	\dualpenalty(\matrix) := \sup\big\{  \inprod{\matrix}{\matrix'}  \mid \matrix' \in \R^{p\times p}, \mypenalty(\matrix')  \leq  1 \big\} = \max_{i_1, i_2}\abs{\matrix_{i_1i_2}}. 
	\end{equation*}
	Therefore, 
	\begin{align*}
	&P\big(\dualpenalty\big(\sum_{i=1}^{n}\big((\matrixtrue)^{-1} - \data^i(\data^i)^\top\big)\big)\leq \tuningparameter_0\big) \\
	&=P\big(\max_{i_1, i_2}\abs{n\big((\matrixtrue)^{-1}-\frac{1}{n}\sum_{i=1}^{n}\data^i(\data^i)^\top\big)_{i_1 i_2}}\leq \tuningparameter_0\big)\\
	& = P\big(\max_{i_1, i_2}\abs{\big((\matrixtrue)^{-1}-\frac{1}{n}\sum_{i=1}^{n}\data^i(\data^i)^\top\big)_{i_1 i_2}}\leq \frac{\tuningparameter_0}{n}\big)\\
	& = 1- P\big(\max_{i_1, i_2}\abs{\big((\matrixtrue)^{-1}-\frac{1}{n}\sum_{i=1}^{n}\data^i(\data^i)^\top\big)_{i_1 i_2}}> \frac{\tuningparameter_0}{n}\big) \\
	& \geq 1-\sum_{i_1=1}^{p}\sum_{i_2=1, i_2 \neq i_1}^{p}P(\abs{(\matrixtrue)^{-1}-\frac{1}{n}\sum_{i=1}^{n}\data^i(\data^i)^\top\big)_{i_1 i_2}}> \frac{\tuningparameter_0}{n}).
	\end{align*}
	{\it Step~2.} \citet[Lemma 1]{Ravikumar2011high} gives the tail bound of sample covariance matrix as
	\begin{equation*}
	P(\abs{\big(\widehat{\matrix}^{-1} - (\matrixtrue)^{-1}\big)_{i_1i_2}}> \delta) \leq 4\exp\big(-\frac{n\delta^2}{80^2\max_k\big((\matrixtrue)^{-1}\big)_{kk}^2}\big),
	\end{equation*}
	for all $\delta \in  (0, 40\max_k\big((\matrixtrue)^{-1}\big)_{kk})$.
	We are now ready to work on the desired tail bound.
	\begin{equation*}
	\sum_{i_1=1}^{p}\sum_{i_2=1, i_2 \neq i_1}^{p}P(\abs{(\matrixtrue)^{-1}-\frac{1}{n}\sum_{i=1}^{n}\data^i(\data^i)^\top\big)_{i_1 i_2}}> \frac{\tuningparameter_0}{n})
	\leq 4p(p-1)\exp(-\frac{n(\tuningparameter_0/n)^2}{80^2\max_k\big((\matrixtrue)^{-1}\big)_{kk}^2}).
	\end{equation*}
	Plugging in $\tuningparameter_0$ yields
	\begin{align*}
	&\sum_{i_1=1}^{p}\sum_{i_2=1,i_2 \neq i_1}^{p}P(\abs{(\matrixtrue)^{-1}-\frac{1}{n}\sum_{i=1}^{n}\data^i(\data^i)^\top\big)_{i_1, i_2}}> \frac{\tuningparameter_0}{n})\\
	& \leq 4p(p-1)\exp\big(-\frac{n\big(80\max_k\big((\matrixtrue)^{-1}\big)_{kk}\sqrt{n(t^2 + \log\big(p(p-1))\big)}/n\big)^2}{80^2\max_k\big((\matrixtrue)^{-1}\big)_{kk}^2}\big)\\
	& =  1- p(p-1)\cdot 4\exp\big(-t^2 - \log\big(p(p-1)\big)\big) \\
	& = 1-4\exp\big(-t^2\big)
	\end{align*}
	for all $0 < \tuningparameter_0/n < 40\max_k\big((\matrixtrue)^{-1}\big)_{kk}$.
	It can now be readily shown that the desired bound holds for all $0 <t <\sqrt{n/4-\log\big(p(p-1)\big)}$. 
\end{proof}

%% file: Sections/Appendix.tex

\section{Notation and properties of tensor operations}
\label{sec: notation}
We follow the notation for tensor and tensor operations  in~\cite{Kolda2006multilinear,Kolda2009tensor}. A tensor $\mathcal{T} \in \R^{\dimension_1 \times \dots \times \dimension_\tensordimension}$ can simply be seen as a multi-dimensional array ($\mathcal{T}_{i_1, \dots ,i_p}  : i_k\in\{1, \dots \dimension_k\}; k \in\{1, \dots, p\} $). The mode-$k$ fibers of $\mathcal{T}$ are vectors obtained by fixing all indices except the $k$th one; for example, $\mathcal{T}_{i_1, \dots,i_{k-1},:, i_{k+1},\dots, i_p} \in \R^{\dimension_k}$. The $k$th mode matricization of $\mathcal{T}$ is the matrix having the mode-$k$ fibers of $\mathcal{T}$ as columns and is represented by $T_{(k)} \in \R^{\dimension_k \times \big(\dimension_1 \dots \dimension_{k-1}\dimension_{k+1}\dots \dimension_p\big)}$. The mode-$k$ product of tensor~$\mathcal{T}$ and a matrix $C \in \R^{m\times \dimension_k}$ is a tensor defined by $\mathcal{Y} = \mathcal{T} \tensorproduct{k} C \in \R^{\dimension_1 \times \dots \times \dimension_{k-1} \times m \times \dimension_{k+1} \times \dots \times \dimension_p}$. The resulting array $\mathcal{Y}$ is from the inversion of the $k$th mode matricization operation on the matrix $C T_{(k)}$, that is, $Y_{(k)} = C T_{(k)}$. The entries of $\mathcal{Y}$ are given by
\begin{equation*}
	\big( \mathcal{T} \tensorproduct{k} C \big)_{i_1,\dots, i_{k-1}, j,i_{k+1},\dots, i_p}= \sum_{i_k=1}^{\dimension_k}T_{i_1,\dots, i_{k-1}, i_k,i_{k+1}, \dots, i_p}c_{j i_k}
\end{equation*}
for $j \in \{1, \dots, m\}, k \in \{2, \dots, p-1\}, i_l\in \{1, \dots, \dimension_l\}, l \in\{1,\dots,k-1, k+1, \dots, p\}$ (the case of $k \in \{1, p\}$ is similar). 
The Tucker product is an extension of the mode-$k$ product, which is the product of a tensor $\mathcal{T}$ and a list of matrices $E= \{ E_1, \dots, E_p \}$ in which $E_k \in \R^{m_k \times \dimension_k}$. The $(i,j)$th element of $E_k$ is denoted by $\element{k}{ij}$. The tensor product is given by
\begin{equation*}
	\mathcal{T} \times E = \mathcal{T} \tensorproduct{1} E_1 \tensorproduct{2} \dots \tensorproduct{p}E_p,
\end{equation*}
or, elementwise, 
\begin{equation*}
	\big( \mathcal{T} \times E \big)_{j_1,\dots, j_p}=\tensorentrysum T_{i_1, \dots, i_p}\element{1}{j_1i_1}\dots \element{p}{j_pi_p}
\end{equation*}
for $j_k\in\{1, \dots, m_k\}, k\in\{1, \dots, p\}$. 
The inner product of two same-sized tensors $\mathcal{W}, \mathcal{V} \in \R^{\dimension_1\times \dots\times \dimension_p}$ is the sum of the products of same-index elements, that is,
\begin{equation*}
	\inprod{\mathcal{W}}{\mathcal{V}} = \tensorentrysum w_{\tensorentryindex} v_{\tensorentryindex}.
\end{equation*}
The array norm of tensor $\mathcal{T}$ is the inner product of itself and given by $\norm{\mathcal{T}}^2 = \inprod{\mathcal{T}}{\mathcal{T}} = \sum_{i_1}\dots\sum_{i_p}t_{i_1, \dots, i_p}^2$.

\begin{lemma}
	\label{remark: distributive}
	Let tensors $\mathcal{W}, \mathcal{V} \in \R^{\dimension_1 \times  \dots \times \dimension_p}$, list of matrices $E=\{ E_1, \dots, E_p \}$ in which $E_k \in \R^{m_k \times \dimension_k}, k\in\{1, \dots, p\}$, it holds that
	\begin{equation*}
	\big( \mathcal{W} +\mathcal{V} \big) \times E = \mathcal{W}\times E +\mathcal{V} \times E.
	\end{equation*}
\end{lemma}

\begin{proof}[of Lemma~\normalfont{D1}]
	This lemma can be proved by writing out the Tucker product on the left-hand side element-wise. For $k \in\{1, \dots p\} $, and $j_k \in\{1, \dots, m_k\} $, 
	\begin{align*}
		\Big( \big( \mathcal{W} + \mathcal{V} \big)\times E \Big)_{j_1, \dots, j_p} =& \tensorentrysum\big( w+v \big)_{i_1, \dots, i_p}\element{1}{j_1 i_1} \dots \element{p}{j_p i_p} \\
		=& \big(\tensorentrysum w_{i_1, \dots, i_p} \element{1}{j_1 i_1} \dots \element{p}{j_p i_p}\big) +\\
		 &\big( \tensorentrysum v_{i_1, \dots, i_p}\element{1}{j_1 i_1} \dots \element{p}{j_p i_p} \big)\\
		=& \big(\mathcal{W} \times E\big)_{j_1, \dots, j_p} + \big(\mathcal{V} \times E\big)_{j_1, \dots, j_p}. 
	\end{align*} 
	This is equivalent to the relationship $\big( \mathcal{W} +\mathcal{V} \big) \times E = \mathcal{W}\times E +\mathcal{V} \times E$.
\end{proof}

\begin{lemma}
	\label{remark: array normal} For two same-sized tensors $\mathcal{W}, \mathcal{V} \in \R^{\dimension_1 \times \dots \times \dimension_p}$,
	\begin{equation*}
		\arraynorm{\mathcal{W}+\mathcal{V}}=\arraynorm{\mathcal{W}}+\arraynorm{\mathcal{V}}+2\inprod{\mathcal{W}}{\mathcal{V}}.
	\end{equation*}
\end{lemma}

\begin{proof}[of Lemma~\normalfont{D2}]
	By the definition of array norm,
	\begin{align*}
	\arraynorm{\mathcal{W}+\mathcal{V}} =& \tensorentrysum	\big( \mathcal{W}+\mathcal{V} \big)_{\tensorentryindex}^2\\
	=& \tensorentrysum\big( w_{\tensorentryindex}+v_{\tensorentryindex} \big)^2\\
	=& \tensorentrysum w_{\tensorentryindex}^2+\tensorentrysum v_{\tensorentryindex}^2+2\tensorentrysum w_{\tensorentryindex}v_{\tensorentryindex}.
	\end{align*}
	That is, $\arraynorm{\mathcal{W}+\mathcal{V}}=\arraynorm{\mathcal{W}}+\arraynorm{\mathcal{V}}+2\inprod{\mathcal{W}}{\mathcal{V}}$.
\end{proof}

%% file: GuaranteesMRLE.bbl
\begin{thebibliography}{60}
\expandafter\ifx\csname natexlab\endcsname\relax\def\natexlab#1{#1}\fi

\bibitem[{Akdemir \& Gupta(2011)}]{Akdemir2011array}
\textsc{Akdemir, D.} \& \textsc{Gupta, A.} (2011).
\newblock Array {V}ariate {R}andom {V}ariables with {M}ultiway {K}ronecker
  {D}elta {C}ovariance {M}atrix {S}tructure.
\newblock \textit{J. Algebr. Stat.} \textbf{2}, 98--113.

\bibitem[{Aravkin et~al.(2017)Aravkin, Burke, Drusvyatskiy, Friedlander \&
  MacPhee}]{Aravkin:2017ur}
\textsc{Aravkin, A.}, \textsc{Burke, J.}, \textsc{Drusvyatskiy, D.},
  \textsc{Friedlander, M.} \& \textsc{MacPhee, K.} (2017).
\newblock {Foundations of Gauge and Perspective Duality}.
\newblock \textit{arXiv:1702.08649} .

\bibitem[{Bercu et~al.(2015)Bercu, Delyon \& Rio}]{bercu2015concentration}
\textsc{Bercu, B.}, \textsc{Delyon, B.} \& \textsc{Rio, E.} (2015).
\newblock \textit{{Concentration Inequalities for Sums and Martingales}}.
\newblock Springer.

\bibitem[{Berk(1972)}]{Berk1972consistency}
\textsc{Berk, R.} (1972).
\newblock Consistency and {A}symptotic {N}ormality of {MLE}'s for {E}xponential
  {M}odels.
\newblock \textit{Ann. Math. Stat.} \textbf{43}, 193--204.

\bibitem[{Bogdan et~al.(2015)Bogdan, van~den Berg, Sabatti, Su \&
  Cand{\`e}s}]{Bogdan2015slope}
\textsc{Bogdan, M.}, \textsc{van~den Berg, E.}, \textsc{Sabatti, C.},
  \textsc{Su, W.} \& \textsc{Cand{\`e}s, E.} (2015).
\newblock {SLOPE-adaptive Variable Selection via Convex Optimization}.
\newblock \textit{Ann. Appl. Stat.} \textbf{9}, 1103--1140.

\bibitem[{Brown(1986)}]{Brown1986fundamentals}
\textsc{Brown, L.} (1986).
\newblock Fundamentals of {S}tatistical {E}xponential {F}amilies with
  {A}pplications in {S}tatistical {D}ecision {T}heory.
\newblock \textit{Lecture Notes-Monograph Series} \textbf{9}, i--279.

\bibitem[{Brush(1967)}]{Brush67}
\textsc{Brush, S.} (1967).
\newblock History of the {L}enz-{I}sing {M}odel.
\newblock \textit{‎Rev. Mod. Phys.} \textbf{39}, 883.

\bibitem[{Bu \& Lederer(2017)}]{Bu:2017tn}
\textsc{Bu, Y.} \& \textsc{Lederer, J.} (2017).
\newblock {Integrating Additional Knowledge Into Estimation of Graphical
  Models}.
\newblock \textit{arXiv:1704.02739v2} .

\bibitem[{B{\"u}hlmann(2013)}]{buhlmann2013statistical}
\textsc{B{\"u}hlmann, P.} (2013).
\newblock Statistical {S}ignificance in {H}igh-dimensional {L}inear {M}odels.
\newblock \textit{Bernoulli} \textbf{19}, 1212--1242.

\bibitem[{B{\"u}hlmann \& van~de Geer(2011)}]{Buhlmann2011statistics}
\textsc{B{\"u}hlmann, P.} \& \textsc{van~de Geer, S.} (2011).
\newblock \textit{{Statistics for High-dimensional Data: Methods, Theory and
  Applications}}.
\newblock Springer.

\bibitem[{Bunea et~al.(2007{\natexlab{a}})Bunea, Tsybakov \& Wegkamp}]{BTW07a}
\textsc{Bunea, F.}, \textsc{Tsybakov, A.} \& \textsc{Wegkamp, M.}
  (2007{\natexlab{a}}).
\newblock {Aggregation for Gaussian Regression}.
\newblock \textit{Ann. Statist.} \textbf{35}, 1674--1697.

\bibitem[{Bunea et~al.(2007{\natexlab{b}})Bunea, Tsybakov \&
  Wegkamp}]{Bunea2007sparsity}
\textsc{Bunea, F.}, \textsc{Tsybakov, A.} \& \textsc{Wegkamp, M.}
  (2007{\natexlab{b}}).
\newblock {Sparsity Oracle Inequalities for the Lasso}.
\newblock \textit{Electron. J. Stat.} \textbf{1}, 169--194.

\bibitem[{Chatterjee(2013)}]{chatterjee2013assumptionless}
\textsc{Chatterjee, S.} (2013).
\newblock Assumptionless {C}onsistency of the {L}asso.
\newblock \textit{arXiv:1303.5817} .

\bibitem[{Chatterjee(2014)}]{chatterjee2014new}
\textsc{Chatterjee, S.} (2014).
\newblock A {N}ew {P}erspective on {L}east {S}quares {U}nder {C}onvex
  {C}onstraint.
\newblock \textit{Ann. Statist.} \textbf{42}, 2340--2381.

\bibitem[{Dalalyan et~al.(2017)Dalalyan, Hebiri \&
  Lederer}]{Dalalyan2017prediction}
\textsc{Dalalyan, A.}, \textsc{Hebiri, M.} \& \textsc{Lederer, J.} (2017).
\newblock On the {P}rediction {P}erformance of the {L}asso.
\newblock \textit{Bernoulli} \textbf{23}, 552--581.

\bibitem[{Dalalyan \& Tsybakov(2007)}]{DT07}
\textsc{Dalalyan, A.} \& \textsc{Tsybakov, A.} (2007).
\newblock {Aggregation by Exponential Weighting and Sharp Oracle Inequalities}.
\newblock In \textit{Learning theory}, vol. 4539. pp. 97--111.

\bibitem[{Dalalyan \& Tsybakov(2012{\natexlab{a}})}]{DT12a}
\textsc{Dalalyan, A.} \& \textsc{Tsybakov, A.} (2012{\natexlab{a}}).
\newblock {Mirror Averaging with Sparsity Priors}.
\newblock \textit{Bernoulli} \textbf{18}, 914--944.

\bibitem[{Dalalyan \& Tsybakov(2012{\natexlab{b}})}]{DT12b}
\textsc{Dalalyan, A.} \& \textsc{Tsybakov, A.} (2012{\natexlab{b}}).
\newblock {Sparse Regression Learning by Aggregation and Langevin Monte-Carlo}.
\newblock \textit{J. Comput. System Sci.} \textbf{78}, 1423--1443.

\bibitem[{De~Lathauwer et~al.(2000)De~Lathauwer, De~Moor \&
  Vandewalle}]{De2000multilinear}
\textsc{De~Lathauwer, L.}, \textsc{De~Moor, B.} \& \textsc{Vandewalle, J.}
  (2000).
\newblock A {M}ultilinear {S}ingular {V}alue {D}ecomposition.
\newblock \textit{SIAM J. Matrix Anal. Appl.} \textbf{21}, 1253--1278.

\bibitem[{Foucart \& Lai(2009)}]{Foucart2009sparsest}
\textsc{Foucart, S.} \& \textsc{Lai, M.} (2009).
\newblock {Sparsest Solutions of Underdetermined Linear Systems via
  $\ell_q$-minimization for $0<q\leq 1$}.
\newblock \textit{Appl. Comput. Harmon. Anal.} \textbf{26}, 395--407.

\bibitem[{Foygel \& Srebro(2011)}]{Foygel2011fast}
\textsc{Foygel, R.} \& \textsc{Srebro, N.} (2011).
\newblock Fast-rate and {O}ptimistic-rate {E}rror {B}ounds for
  $\ell_1$-regularized {R}egression.
\newblock \textit{arXiv:1108.0373} .

\bibitem[{Friedlander \& Mac{\^e}do(2016)}]{Friedlander:2016ge}
\textsc{Friedlander, M.} \& \textsc{Mac{\^e}do, I.} (2016).
\newblock {Low-rank Spectral Optimization via Gauge Duality}.
\newblock \textit{SIAM J. Sci. Comput.} \textbf{38}, A1616--A1638.

\bibitem[{Friedman et~al.(2008)Friedman, Hastie \&
  Tibshirani}]{Friedman2008sparse}
\textsc{Friedman, J.}, \textsc{Hastie, T.} \& \textsc{Tibshirani, R.} (2008).
\newblock Sparse {I}nverse {C}ovariance {E}stimation with the {G}raphical
  {L}asso.
\newblock \textit{Biostatistics} \textbf{9}, 432--441.

\bibitem[{Giraud(2014)}]{Giraud2014introduction}
\textsc{Giraud, C.} (2014).
\newblock \textit{Introduction to {H}igh-dimensional {S}tatistics}.
\newblock CRC Press.

\bibitem[{Gramfort et~al.(2012)Gramfort, Kowalski \&
  H{\"a}m{\"a}l{\"a}inen}]{Gramfort2012mixed}
\textsc{Gramfort, A.}, \textsc{Kowalski, M.} \& \textsc{H{\"a}m{\"a}l{\"a}inen,
  M.} (2012).
\newblock {Mixed-norm Estimates for the M/EEG Inverse Problem Using Accelerated
  Gradient Methods}.
\newblock \textit{Phys. Med. Biol.} \textbf{57}, 1937--1961.

\bibitem[{Greenshtein \& Ritov(2004)}]{greenshtein2004persistence}
\textsc{Greenshtein, E.} \& \textsc{Ritov, Y.} (2004).
\newblock {Persistence in High-dimensional Linear Predictor Selection and the
  Virtue of Overparametrization}.
\newblock \textit{Bernoulli} \textbf{10}, 971--988.

\bibitem[{Gu et~al.(2015)Gu, Cao, Ning \& Liu}]{Gu2015local}
\textsc{Gu, Q.}, \textsc{Cao, Y.}, \textsc{Ning, Y.} \& \textsc{Liu, H.}
  (2015).
\newblock Local and {G}lobal {I}nference for {H}igh {D}imensional
  {N}onparanormal {G}raphical {M}odels.
\newblock \textit{arXiv:1502.02347} .

\bibitem[{Hastie et~al.(2015)Hastie, Tibshirani \&
  Wainwright}]{Hastie2015statistical}
\textsc{Hastie, T.}, \textsc{Tibshirani, R.} \& \textsc{Wainwright, M.} (2015).
\newblock \textit{Statistical {L}earning with {S}parsity: {T}he {L}asso and
  {G}eneralizations}.
\newblock CRC press.

\bibitem[{Hebiri \& Lederer(2013)}]{Hebiri13}
\textsc{Hebiri, M.} \& \textsc{Lederer, J.} (2013).
\newblock How {C}orrelations {I}nfluence {L}asso {P}rediction.
\newblock \textit{IEEE Trans. Inf. Theory} \textbf{59}, 1846--1854.

\bibitem[{Hoff(2011)}]{Hoff2011separable}
\textsc{Hoff, P.} (2011).
\newblock Separable {C}ovariance {A}rrays via the {T}ucker {P}roduct, with
  {A}pplications to {M}ultivariate {R}elational {D}ata.
\newblock \textit{Bayesian Anal.} \textbf{6}, 179--196.

\bibitem[{Huang \& Zhang(2012)}]{Huang:2012ww}
\textsc{Huang, J.} \& \textsc{Zhang, C.} (2012).
\newblock {Estimation and Selection via Absolute Penalized Convex Minimization
  and its Multistage Adaptive Applications}.
\newblock \textit{J. Mach. Learn. Res.} \textbf{13}, 1839--1864.

\bibitem[{Huntsberger \& Billingsley(1981)}]{Huntsberger1981elements}
\textsc{Huntsberger, D.} \& \textsc{Billingsley, P.} (1981).
\newblock \textit{Elements of {S}tatistical {I}nference (fifth ed.)}.
\newblock Allyn Bacon.

\bibitem[{James \& Stein(1961)}]{James1961estimation}
\textsc{James, W.} \& \textsc{Stein, C.} (1961).
\newblock Estimation with {Q}uadratic {L}oss.
\newblock In \textit{Proceedings of the Fourth Berkeley Symposium on
  Mathematical Statistics and Probability}, vol.~1. pp. 361--379.

\bibitem[{Johansen(1979)}]{Johansen1979introduction}
\textsc{Johansen, S.} (1979).
\newblock \textit{Introduction to the {T}heory of {R}egular {E}xponential
  {F}amilies}.
\newblock Lecture Notes, Institute of Mathematical Statistics, University of
  Copenhagen.

\bibitem[{Kolda(2006)}]{Kolda2006multilinear}
\textsc{Kolda, T.} (2006).
\newblock {M}ultilinear {O}perators for {H}igher-order {D}ecompositions.
\newblock Tech. Rep. SAND2006-2081, Sandia National Laboratories.

\bibitem[{Kolda \& Bader(2009)}]{Kolda2009tensor}
\textsc{Kolda, T.} \& \textsc{Bader, B.} (2009).
\newblock Tensor {D}ecompositions and {A}pplications.
\newblock \textit{SIAM rev.} \textbf{51}, 455--500.

\bibitem[{Koltchinskii et~al.(2011)Koltchinskii, Lounici \&
  Tsybakov}]{Koltchinskii2011nuclear}
\textsc{Koltchinskii, V.}, \textsc{Lounici, K.} \& \textsc{Tsybakov, A.}
  (2011).
\newblock {Nuclear-norm Penalization and Optimal Rates for Noisy Low-rank
  Matrix Completion}.
\newblock \textit{Ann. Statist.} \textbf{39}, 2302--2329.

\bibitem[{Lederer(2016)}]{Lederer2016graphical}
\textsc{Lederer, J.} (2016).
\newblock Graphical {M}odels for {D}iscrete and {C}ontinuous {D}ata.
\newblock \textit{arXiv:1609.05551} .

\bibitem[{Lederer et~al.(2016)Lederer, Yu \& Gaynanova}]{Lederer2016oracle}
\textsc{Lederer, J.}, \textsc{Yu, L.} \& \textsc{Gaynanova, I.} (2016).
\newblock Oracle {I}nequalities for {H}igh-dimensional {P}rediction.
\newblock \textit{arXiv:1608.00624} .

\bibitem[{Lenz(1920)}]{Lenz20}
\textsc{Lenz, W.} (1920).
\newblock Beitr{\"a}ge zum {V}erst{\"a}ndnis der magnetischen {E}igenschaften
  in festen {K}{\"o}rpern.
\newblock \textit{Physikalische Zeitschrift} \textbf{21}, 613--615.

\bibitem[{Li et~al.(2013)Li, Zhou \& Li}]{Li:2013tm}
\textsc{Li, X.}, \textsc{Zhou, H.} \& \textsc{Li, L.} (2013).
\newblock {Tucker Tensor Regression and Neuroimaging Analysis}.
\newblock \textit{arXiv:1304.5637} .

\bibitem[{Massart \& Meynet(2011)}]{MM11}
\textsc{Massart, P.} \& \textsc{Meynet, C.} (2011).
\newblock The {L}asso as an $\ell_1$-ball {M}odel {S}election {P}rocedure.
\newblock \textit{Electron. J. Stat.} \textbf{5}, 669--687.

\bibitem[{Rabusseau \& Kadri(2016)}]{Rabusseau:2016um}
\textsc{Rabusseau, G.} \& \textsc{Kadri, H.} (2016).
\newblock {Low-rank Regression with Tensor Responses}.
\newblock In \textit{Advances in Neural Information Processing Systems 29}. pp.
  1867--1875.

\bibitem[{Raskutti et~al.(2015)Raskutti, Yuan \& Chen}]{Raskutti2015convex}
\textsc{Raskutti, G.}, \textsc{Yuan, M.} \& \textsc{Chen, H.} (2015).
\newblock {Convex {R}egularization for {H}igh-dimensional {M}ulti-response
  {T}ensor {R}egression}.
\newblock \textit{arXiv:1512.01215} .

\bibitem[{Ravikumar et~al.(2011)Ravikumar, Wainwright, Raskutti \&
  Yu}]{Ravikumar2011high}
\textsc{Ravikumar, P.}, \textsc{Wainwright, M.}, \textsc{Raskutti, G.} \&
  \textsc{Yu, B.} (2011).
\newblock High-dimensional {C}ovariance {E}stimation by {M}inimizing
  $\ell_1$-penalized {L}og-determinant {D}ivergence.
\newblock \textit{Electron. J. Stat.} \textbf{5}, 935--980.

\bibitem[{Rigollet \& Tsybakov(2011)}]{RigTsy11}
\textsc{Rigollet, P.} \& \textsc{Tsybakov, A.} (2011).
\newblock {E}xponential {S}creening and {O}ptimal {R}ates of {S}parse
  {E}stimation.
\newblock \textit{Ann. Statist.} \textbf{39}, 731--771.

\bibitem[{Rothman et~al.(2008)Rothman, Bickel, Levina \&
  Zhu}]{Rothman2008sparse}
\textsc{Rothman, A.}, \textsc{Bickel, P.}, \textsc{Levina, E.} \& \textsc{Zhu,
  J.} (2008).
\newblock Sparse {P}ermutation {I}nvariant {C}ovariance {E}stimation.
\newblock \textit{Electron. J. Stat.} \textbf{2}, 494--515.

\bibitem[{Shevlyakova \& Morgenthaler(2013)}]{Shevlyakova2012KLD}
\textsc{Shevlyakova, M.} \& \textsc{Morgenthaler, S.} (2013).
\newblock Identifying {G}raphical {M}odels.
\newblock \textit{arXiv:1309.5740} .

\bibitem[{Simon et~al.(2013)Simon, Friedman, Hastie \&
  Tibshirani}]{Simon2013sparse}
\textsc{Simon, N.}, \textsc{Friedman, J.}, \textsc{Hastie, T.} \&
  \textsc{Tibshirani, R.} (2013).
\newblock A {S}parse-group {L}asso.
\newblock \textit{J. Comput. Graph. Statist.} \textbf{22}, 231--245.

\bibitem[{Sun \& Li(2016)}]{Sun:2016uj}
\textsc{Sun, W.} \& \textsc{Li, L.} (2016).
\newblock {Sparse Tensor Response Regression and Neuroimaging Analysis}.
\newblock \textit{arXiv:1609.04523} .

\bibitem[{Tibshirani(1996)}]{Tibshirani1996regression}
\textsc{Tibshirani, R.} (1996).
\newblock Regression {S}hrinkage and {S}election via the {L}asso.
\newblock \textit{J. R. Stat. Soc. Ser. B. Stat. Methodol.} \textbf{58},
  267--288.

\bibitem[{van~de Geer(2008)}]{Van2008high}
\textsc{van~de Geer, S.} (2008).
\newblock {High-dimensional Generalized Linear Models and the Lasso}.
\newblock \textit{Ann. Statist.} \textbf{36}, 614--645.

\bibitem[{van~de Geer(2016)}]{Van2016estimation}
\textsc{van~de Geer, S.} (2016).
\newblock \textit{{Estimation and Testing Under Sparsity}}.
\newblock Springer.

\bibitem[{Verzelen(2012)}]{Verzelen12}
\textsc{Verzelen, N.} (2012).
\newblock Minimax {R}isks for {S}parse {R}egressions: {U}ltra-high
  {D}imensional {P}henomenons.
\newblock \textit{Electron. J. Stat.} \textbf{6}, 38--90.

\bibitem[{Wainwright \& Jordan(2008)}]{Wainwright08}
\textsc{Wainwright, M.} \& \textsc{Jordan, M.} (2008).
\newblock Graphical {M}odels, {E}xponential {F}amilies, and {V}ariational
  {I}nference.
\newblock \textit{Found. Trends. Machine Learning} \textbf{1}, 1--305.

\bibitem[{Yuan \& Lin(2006)}]{Yuan:2006jy}
\textsc{Yuan, M.} \& \textsc{Lin, Y.} (2006).
\newblock Model {S}election and {E}stimation in {R}egression with {G}rouped
  {V}ariables.
\newblock \textit{J. R. Stat. Soc. Ser. B. Stat. Methodol.} \textbf{68},
  49--67.

\bibitem[{Yuan \& Lin(2007)}]{Yuan2007model}
\textsc{Yuan, M.} \& \textsc{Lin, Y.} (2007).
\newblock Model {S}election and {E}stimation in the {G}aussian {G}raphical
  {M}odel.
\newblock \textit{Biometrika} \textbf{94}, 19--35.

\bibitem[{Zhang et~al.(2017)Zhang, Wainwright \& Jordan}]{Zhang2015optimal}
\textsc{Zhang, Y.}, \textsc{Wainwright, M.} \& \textsc{Jordan, M.} (2017).
\newblock {Optimal Prediction for Sparse Linear Models? Lower Bounds for
  Coordinate-separable M-estimators}.
\newblock \textit{Electron. J. Stat.} \textbf{11}, 752--799.

\bibitem[{Zhou et~al.(2013)Zhou, Li \& Zhu}]{Zhou:2013ec}
\textsc{Zhou, H.}, \textsc{Li, L.} \& \textsc{Zhu, H.} (2013).
\newblock {Tensor Regression with Applications in Neuroimaging Data Analysis}.
\newblock \textit{J. Amer. Statist. Assoc.} \textbf{108}, 540--552.

\bibitem[{Zou(2006)}]{Zou:2006du}
\textsc{Zou, H.} (2006).
\newblock {The Adaptive Lasso and Its Oracle Properties}.
\newblock \textit{J. Amer. Statist. Assoc.} \textbf{101}, 1418--1429.

\end{thebibliography}
